%% file: main.tex
\PassOptionsToPackage{dvipsnames,table,xcdraw}{xcolor}
\documentclass{article} 

\usepackage[final]{colm2026_conference}

\input{utils/comment}
\input{utils/general_utils}

\input{utils/includes}
\input{utils/math_utils}
\input{utils/custom_style}

\usepackage{algorithm}
\usepackage{amsmath}
\usepackage{amssymb}
\usepackage{url}
\usepackage{array}
\usepackage{microtype}
\usepackage{graphicx}
\usepackage{lineno}
\usepackage{subfigure}
\usepackage{bm}
\usepackage{dsfont}
\usepackage{mathtools}
\usepackage{nccmath}
\usepackage{amssymb}
\usepackage{multirow}
\usepackage{booktabs}
\usepackage{varwidth}
\usepackage{pifont}
\usepackage{makecell}
\usepackage{wrapfig}
\usepackage{varwidth}
\usepackage{makecell}
\usepackage{enumitem}
\usepackage{adjustbox}
\usepackage{graphicx,calc}
\usepackage{amsfonts,amsthm,bm}
\usepackage[flushleft]{threeparttable}
\usepackage{pifont}
\usepackage{ulem}
\usepackage{tabularx}
\usepackage{arydshln}
\usepackage{float}
\tcbuselibrary{breakable,theorems,skins}
\usepackage{cleveref}

\usepackage{pgfmath}
\usepackage{siunitx}
\usepackage{listings}
\title{Capacity-Dependent Effects of Data Selection for Reasoning}

\author{Cuong Dang \& Hoang Anh Just \& Ruoxi Jia \\
Department of Electrical and Computer Engineering\\
Virginia Tech\\
Blacksburg, VA, USA \\
\texttt{\{cuongdc,just,ruoxijia\}@vt.edu} \\
}

\begin{document}

\definecolor{SMALLCOLOR}{HTML}{38768D}
\definecolor{LARGECOLOR}{HTML}{AC6055}

\ifcolmsubmission
\linenumbers
\fi

\maketitle

\begin{figure}[ht]
    \centering
    \includegraphics[width=\linewidth]{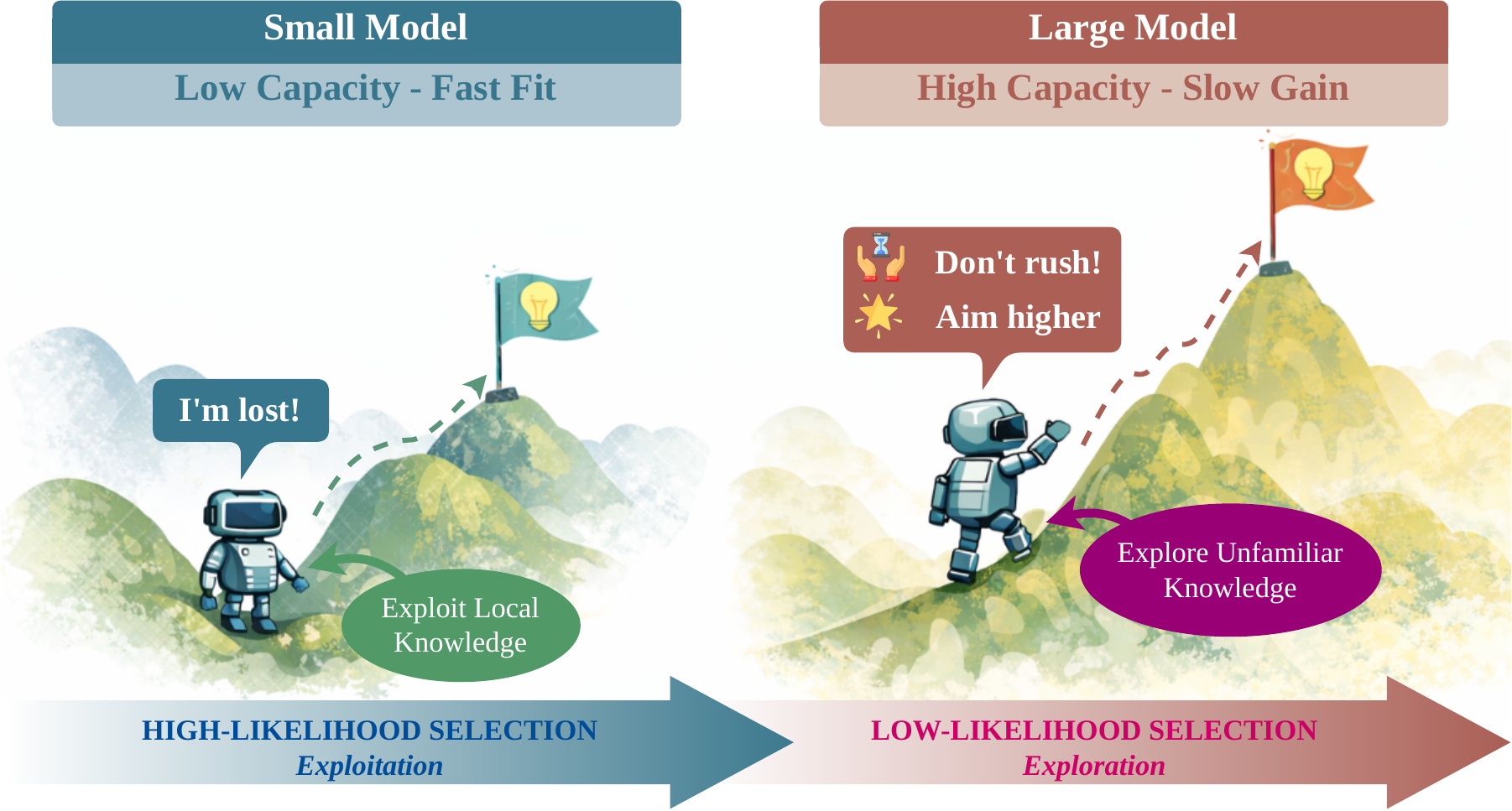}
    \caption{\textbf{Capacity-dependent learning dynamics under likelihood-based data selection.}
The illustration conceptualizes the optimization landscape faced by models with different capacities when learning reasoning tasks. \textbf{\textcolor{SMALLCOLOR}{A small model}} (\textbf{\textcolor{SMALLCOLOR}{left}}) can only climb to a nearby local peak, representing limited capability to learn from distant supervision signals. When trained on low-likelihood examples---responses far from its current policy---it struggles to escape local minima and therefore benefits from high-likelihood data, which allows it to exploit existing knowledge and achieve incremental improvements. In contrast, \textbf{\textcolor{LARGECOLOR}{a large model}} (\textbf{\textcolor{LARGECOLOR}{right}}) possesses sufficient representational capacity and optimization flexibility to traverse a longer path across the landscape. Although learning from low-likelihood data may require longer training, the model can ultimately reach a higher global peak, approaching the teacher distribution and acquiring stronger reasoning ability.}
    \label{fig:thumbnail}
\end{figure}

\input{sections/Abstract}

\input{sections/Introduction}

\input{sections/Related_Works}

\input{sections/Problem_Setup}

\input{sections/Empirical_Results}

\input{sections/Deep_Analysis}

\input{sections/Theory}

\input{sections/Conclusion}

\input{sections/Acknowledgement}

\bibliography{colm2026_conference}
\bibliographystyle{colm2026_conference}

\input{sections/Appendix}

\end{document}

%% file: utils/general_utils.tex
\usepackage{pifont}

\usepackage{color, colortbl}
\definecolor{Gray}{gray}{0.93}
\definecolor{Orange}{rgb}{1,0.5,0}
\definecolor{DGray}{gray}{0.83}
\definecolor{LightCyan}{rgb}{0.88,1,1}
\definecolor{pastelgreen}{HTML}{6fc276}
\definecolor{pastelred}{HTML}{ff746c}
\definecolor{pastelpurple}{HTML}{c9a0dc}
\definecolor{pastelorange}{HTML}{ff964f}

\usepackage[T1]{fontenc}

%% file: utils/includes.tex
\usepackage{microtype}
\usepackage{graphicx}
\usepackage{subfigure}
\usepackage{booktabs}

\usepackage{wrapfig}
\usepackage{pifont}

\usepackage{blindtext}
\usepackage{lipsum}

\usepackage{multirow}
\usepackage{graphicx}
\usepackage{listings}

\usepackage{bbm}

\usepackage [english]{babel}
\usepackage [autostyle, english = american]{csquotes}
\usepackage{amsmath}
\usepackage{amssymb}
\usepackage{mathtools}
\usepackage{amsthm}

\definecolor{mycitecolor}{HTML}{3498DC}
\definecolor{mylinkcolor}{HTML}{E74D3B}
\definecolor{myurlcolor}{HTML}{980000}
\definecolor{mydarkgreen}{HTML}{6a994e}
\definecolor{myorange}{HTML}{E7730D}
\definecolor{myblue}{HTML}{4594c1}
\usepackage[colorlinks=true,linkcolor=mylinkcolor,citecolor=mycitecolor,urlcolor=myurlcolor]{hyperref}

\usepackage[capitalize,noabbrev]{cleveref}
\usepackage{fontawesome}

\theoremstyle{plain}
\newtheorem{theorem}{Theorem}[section]
\newtheorem{proposition}[theorem]{Proposition}
\newtheorem{lemma}[theorem]{Lemma}

\theoremstyle{definition}

\theoremstyle{remark}

\usepackage[textsize=tiny]{todonotes}
\usepackage{thmtools, thm-restate}

\usepackage{pifont}
\usepackage{url}
\usepackage[most]{tcolorbox}
\usepackage{lipsum}
\usepackage{wrapfig}
\usepackage{booktabs}
\usepackage{multirow,mathtools } 

\usepackage{adjustbox}
\MakeOuterQuote{"}

%% file: utils/math_utils.tex
\usepackage{amsmath,amsfonts,bm}

\def\eqref#1{(\ref{#1})}

\def\1{\bm{1}}

\DeclareMathAlphabet{\mathsfit}{\encodingdefault}{\sfdefault}{m}{sl}
\SetMathAlphabet{\mathsfit}{bold}{\encodingdefault}{\sfdefault}{bx}{n}

\def\dd{{\textrm{d}}}

\DeclareMathOperator*{\argmax}{arg\,max}
\DeclareMathOperator*{\argmin}{arg\,min}

\newcommand{\bx}{\mathbf{x}}

\newcommand{\p}[1]{\left(#1\right)}      

%% file: utils/custom_style.tex
\usepackage{color, colortbl}
\usepackage{xcolor}
\usepackage{colortbl}
\definecolor{mycitecolor}{HTML}{3498DC}
\definecolor{mylinkcolor}{HTML}{E74D3B}
\definecolor{myurlcolor}{HTML}{980000}
\definecolor{mydarkgreen}{HTML}{6a994e}
\definecolor{myorange}{HTML}{E7730D}
\definecolor{myblue}{HTML}{4594c1}
\definecolor{myyellow}{HTML}{FFDF00}
\definecolor{mypurple}{HTML}{C025CB}
\definecolor{RQframe}{HTML}{CFAB8D}
\definecolor{URAGColor}{HTML}{DE1A58}

\definecolor{sparse1}{HTML}{019E4F}
\definecolor{sparse2}{HTML}{E94127}
\definecolor{sparse3}{HTML}{017DC7}
\definecolor{dense1}{HTML}{3B6C73}
\definecolor{dense2}{HTML}{C1565E}
\definecolor{dense3}{HTML}{4682B4}

\definecolor{Appendixcolor}{HTML}{E74D3B}
\definecolor{Fgcolor}{HTML}{DD0303}
\definecolor{Sectioncolor}{HTML}{E74D3B}
\definecolor{Subsectioncolor}{HTML}{E74D3B}
\definecolor{Equationcolor}{HTML}{E74D3B}
\definecolor{Tablecolor}{HTML}{E74D3B}
\definecolor{Observationcolor}{HTML}{6a994e}
\definecolor{Takeawaycolor}{HTML}{9381ff}
\definecolor{Algcolor}{HTML}{9305f2}
\definecolor{Promptcolor}{HTML}{FF6C0C}
\definecolor{RQcolor}{HTML}{001BB7}
\definecolor{Theocolor}{HTML}{0744fa}
\definecolor{Lemmacolor}{HTML}{ea00ff}
\definecolor{Propositioncolor}{HTML}{de07fa}
\definecolor{Inequalitycolor}{HTML}{ff9400}
\definecolor{Assumptioncolor}{HTML}{E76F2E}

\newcommand{\Appendix}{\textcolor{Appendixcolor}{Appendix}}
\newcommand{\Figure}{\textcolor{Fgcolor}{Figure}}
\newcommand{\Section}{\textcolor{Sectioncolor}{Section}}

\newcommand{\Equation}{\textcolor{Equationcolor}{Equation}}
\newcommand{\Table}{\textcolor{Tablecolor}{Table}}

\newcommand{\Theorem}{\textcolor{Theocolor}{Theorem}}
\newcommand{\Lemma}{\textcolor{Lemmacolor}{Lemma}}
\newcommand{\Proposition}{\textcolor{Propositioncolor}{Proposition}}
\newcommand{\Inequality}{\textcolor{Inequalitycolor}{Inequality}}
\newcommand{\Assumption}{\textcolor{Assumptioncolor}{Assumption}}

\usepackage[normalem]{ulem}
\usepackage{enumitem}
\usepackage{titletoc}

\newcommand{\Appendixref}[1]{\Appendix~{\hypersetup{linkcolor=Appendixcolor}\ref{#1}}}
\newcommand{\Figureref}[1]{\Figure~{\hypersetup{linkcolor=Fgcolor}\ref{#1}}}
\newcommand{\Sectionref}[1]{\Section~{\hypersetup{linkcolor=Sectioncolor}\ref{#1}}}

\newcommand{\Equationref}[1]{\Equation~{\hypersetup{linkcolor=Equationcolor}\ref{#1}}}
\newcommand{\Theoremref}[1]{\Theorem~{\hypersetup{linkcolor=Theocolor}\ref{#1}}}
\newcommand{\Lemmaref}[1]{\Lemma~{\hypersetup{linkcolor=Lemmacolor}\ref{#1}}}
\newcommand{\Propositionref}[1]{\Proposition~{\hypersetup{linkcolor=Propositioncolor}\ref{#1}}}
\newcommand{\Tableref}[1]{\Table~{\hypersetup{linkcolor=Tablecolor}\ref{#1}}}

\newcommand{\Assumptionref}[1]{\Assumption~{\hypersetup{linkcolor=Assumptioncolor}\ref{#1}}}
\newcommand{\Inequalityref}[1]{\Inequality~{\hypersetup{linkcolor=Inequalitycolor}\ref{#1}}}

\usepackage[most]{tcolorbox}

\newcounter{rq}

\newcounter{takeaway}
\renewcommand{\thetakeaway}{\arabic{takeaway}}

\newenvironment{takeaway}[1][]%
{%
    \refstepcounter{takeaway}
    \tcolorbox[
        enhanced,
        colback=white,
        colframe=white,
        leftrule=0.4mm,
        rightrule=0.4mm,
        toprule=0.4mm,
        bottomrule=0.4mm,
        arc=0mm,
        left=0pt,
        right=0pt,
        top=2pt,
        bottom=2pt,
        breakable,
        borderline north={0.4mm}{0pt}{Takeawaycolor},
        borderline south={0.4mm}{0pt}{Takeawaycolor}
    ]
    \textbf{\textcolor{Takeawaycolor}{\textit{Takeaway~\thetakeaway}}}
    \ifx\relax#1\relax\else~(\textit{#1}).\fi%
}
{%
    \endtcolorbox
}

\newcounter{observation}
\renewcommand{\theobservation}{\arabic{observation}}

\newenvironment{observation}[1][]%
{%
    \refstepcounter{observation}
    \tcolorbox[
        enhanced,
        colback=white,
        colframe=white,
        leftrule=0.4mm,
        rightrule=0.4mm,
        toprule=0.4mm,
        bottomrule=0.4mm,
        arc=0mm,
        left=0pt,
        right=0pt,
        top=2pt,
        bottom=2pt,
        breakable,
        borderline north={0.4mm}{0pt}{mydarkgreen!80!black},
        borderline south={0.4mm}{0pt}{mydarkgreen!80!black}
    ]
    \textbf{\textcolor{mydarkgreen!80!black}{\textit{Observation~\theobservation}}}
    \ifx\relax#1\relax\else~(\textit{#1}).\fi%
}
{%
    \endtcolorbox
}

\newcounter{promptbox}

\newcommand{\keyquestion}[1]{%
\begin{tcolorbox}[
    enhanced,
    colback=myblue!8!white,
    colframe=myblue,
    leftrule=2mm,
    rightrule=0mm,
    toprule=0mm,
    bottomrule=0mm,
    arc=0mm,
    left=5pt,
    right=5pt,
    top=5pt,
    bottom=5pt,
    breakable,
    leftlower=2mm,
    leftupper=2mm,
]
#1
\end{tcolorbox}
}

\newcommand{\contribution}[1]{%
\begin{tcolorbox}[
    enhanced,
    colback=mypurple!8!white,
    colframe=mypurple,
    leftrule=2mm,
    rightrule=0mm,
    toprule=0mm,
    bottomrule=0mm,
    arc=0mm,
    left=5pt,
    right=5pt,
    top=5pt,
    bottom=5pt,
    breakable,
    leftlower=2mm,
    leftupper=2mm
]
\normalsize 
\textit{#1}
\end{tcolorbox}
}

%% file: sections/Abstract.tex
\begin{abstract}
In reasoning supervised fine-tuning, candidate responses for the same instruction can differ substantially in how well they match the student’s current distribution. Recent likelihood-based response selection methods suggest that responses closer to the student distribution provide more effective supervision, motivating the hypothesis that high-likelihood responses may generally be preferable for fine-tuning. In this paper, we revisit this intuition and show that the value of likelihood-based data selection depends critically on model capacity and training duration. Through controlled experiments on mathematical reasoning, using students ranging from 1.5B to 8B parameters and supervision generated by stronger teacher models, we observe a clear \emph{capacity-dependent} ``{\color{SMALLCOLOR}\textbf{Fast-Fit}} / {\color{LARGECOLOR}\textbf{Slow-Gain}}'' pattern. High-likelihood data provides faster and more stable early improvements, especially for smaller models, but low-likelihood data becomes increasingly beneficial for larger models when training is allowed to continue longer. To explain this phenomenon, we analyze learning dynamics, showing that small models often fail to absorb low-likelihood supervision and instead fall into shallow or repetitive behaviors, while larger models are better able to move toward the teacher distribution under such data. We further provide a capacity-constrained theoretical view of distillation that clarifies how data difficulty, data span, and student capacity jointly govern transfer. Overall, our findings show that effective data selection for reasoning should be aware of model capacity and computing budget rather than based on a single universal preference for high-likelihood supervision.
\end{abstract}

%% file: sections/Introduction.tex
\section{Introduction}

Large language models (LLMs) are trained on massive text corpora to predict the next token, enabling them to generate fluent language and perform a wide range of tasks such as reasoning, coding, and question answering. After this broad pretraining stage, a common next step is supervised fine-tuning (SFT), where the model is trained on curated instruction-response pairs to sharpen useful behaviors and elicit latent capabilities.

In SFT for reasoning, supervision is highly heterogeneous: even for the same instruction, candidate responses can differ substantially in how well they match the student’s current distribution. Some responses are already close to what the model can produce, while others are much harder for it to reproduce. Choosing which response to train on for each instruction therefore directly affects optimization stability, sample efficiency, and the kind of reasoning behavior the student ultimately acquires. 

Recent work has begun to explore this problem of response selection explicitly. In particular, GRAPE~\citep{zhang2025the} proposes a likelihood-based response selection strategy: for each instruction, it selects the response with the highest probability under the target model from a pool of candidate answers. We focus on GRAPE in particular because it is one of the earliest works to study response selection in instruction tuning, and because its simple likelihood-based rule, which incurs substantially lower overhead than alternatives~\citep{xia2024less,yang2024smalltolarge}, was shown to be effective in the settings originally studied. The underlying intuition of GRAPE is straightforward: supervision should respect the learner's current distribution, and examples that are already more likely under the student may be easier to absorb and thus more effective for fine-tuning. At the same time, a long and influential line of work across active learning and data selection advocates for the \textit{opposite} philosophy---that models learn most from \textit{hard} examples, i.e., those with low log-likelihood under the current model~\citep{lewis1995sequential,shrivastava2016training,robinsoncontrastive,paul2021deep,lin2024not}. From this perspective, easy, high-likelihood samples are precisely the ones to \textit{avoid}, as they contribute little gradient signal and fail to push the model beyond its current capability boundary.

\keyquestion{
Hence, this raises \textbf{fundamental questions}: \textbf{\ding{182}} \emph{Does high-likelihood data always work best?} (\Sectionref{section:validate}), \textbf{\ding{183}} \emph{Why} does it help in some settings but fail in others? (\Sectionref{sec:learning_dynamics_analysis}), \textbf{\ding{184}} \emph{What} is the \emph{general mechanism} behind likelihood-based data selection? (\Sectionref{sec:theory})
}

In this work, we study these questions through the lens of model capacity and learning dynamics. Our central hypothesis is that the value of data selection depends critically on whether the student model has sufficient capacity to benefit from challenging supervision. Through controlled experiments across model scales, as shown in \Figureref{fig:thumbnail}, we observe a clear capacity-dependent pattern. For small models, selecting low-likelihood data is often ineffective: the model fails to move meaningfully toward the teacher distribution, exhibits unstable training dynamics, and tends to generate repetitive or shallow outputs rather than genuinely improved reasoning. In contrast, larger models can benefit substantially from low-likelihood data, especially when trained long enough, because such data pushes them beyond their current distribution and enables stronger adaptation toward the teacher's reasoning behavior. We also find that training duration matters: high-likelihood data often provides rapid early gains, whereas low-likelihood data may yield slower but ultimately greater improvements for sufficiently capable models. These observations suggest that data selection for reasoning should not be treated as a one-size-fits-all rule. Instead, the effectiveness of high- or low-likelihood supervision is governed by an interaction between data distribution, model capacity, and computing budget. 

\contribution{Overall, our \textbf{contributions are threefold}: \textbf{\ding{182}} We reveal a capacity-dependent effect of likelihood-based data selection for reasoning, showing that high-likelihood data is not universally optimal. \textbf{\ding{183}} We explain this phenomenon through optimization and learning dynamics. \textbf{\ding{184}} We develop a thereotical capacity-constrained view of distillation showing that knowledge transfer is jointly governed by initialization, knowledge gap, data space, and student capacity.}

%% file: sections/Related_Works.tex
\section{Related Works}
\textbf{SFT, Distillation \& Data Engineering for Reasoning.} SFT has emerged as a central mechanism for eliciting reasoning in LLMs, with prior work showing that training on rationales can substantially improve multi-step reasoning. For instance, \citet{zelikman2022star} demonstrate that reasoning can be bootstrapped by iteratively generating and fine-tuning on successful rationales. More recently, the success of DeepSeek~\citep{guo2025deepseek} highlighted the effectiveness of distilling reasoning ability from powerful teacher models into smaller students. This distillation paradigm has since been adopted by a series of open-source projects, including OpenR1~\citep{openr1}, OpenThoughts~\citep{guha2026openthoughts}, AceReason~\citep{liu2026acereasonnemotron}, and NVIDIA OpenMathReasoning~\citep{moshkov2025aimo2}. Yet, given a large pool of teacher-generated responses, an important question is which responses are most useful for student learning. \textit{On-Policy Distillation}~\citep{agarwal2024onpolicy} argues that supervision closer to the student’s own generation distribution is more learnable, and \textit{GRAPE}~\citep{zhang2025the} applies this intuition to instruction tuning by selecting responses that best fit the target model. These works collectively suggest that the effectiveness of reasoning supervision depends not only on teacher quality, but also on its compatibility with the learner’s current distribution. \textit{In this paper, we revisit this hypothesis and show that while high-likelihood data often yields faster and more stable optimization, low-likelihood data can become increasingly beneficial over longer training horizons because it carries knowledge beyond the student’s current capabilities}. 


\paragraph{Context-Dependent Data Selection.} More broadly, our work connects to a line of research showing that there is no universally optimal rule for difficulty-based data selection, and that its effectiveness depends on the learning context. Curriculum and self-paced learning frameworks suggest that the usefulness of harder examples varies with the learner’s current stage, motivating adaptive or easy-to-hard training schedules rather than fixed policies~\citep{bengio2009curriculum,kumar2010self,platanios2019competence,lalor2020dynamic}. This dependence also arises with respect to the data regime: \citet{sorscher2022beyond} show that when data is scarce, retaining easy examples is preferable as they capture coarse-grained structure, whereas when data is abundant, harder examples become more valuable as easy ones grow increasingly redundant. More generally, recent studies suggest that no single data selection strategy consistently dominates across settings, with performance depending on factors such as training time budget, data composition, and representation~\citep{wu2020curricula,du2025disentangling}. Closest to our setting, prior work has also argued that the usefulness of difficult supervision depends on model capacity, but in different settings and with different notions of difficulty. \citet{gao2025principled} study this question in preference learning, where difficulty is defined over preference optimization examples rather than reasoning traces. \citet{li2025small} study reasoning distillation, but operationalize difficulty through proxies such as the length of reasoning traces and teacher capability, showing that small models often benefit more from shorter chain-of-thoughts or weaker teachers. \textit{In contrast, our work focuses on reasoning distillation under likelihood-based data selection, where difficulty is measured directly by the student's likelihood of the teacher response. This provides a student-grounded notion of difficulty, rather than a student-independent proxy such as trace length. Beyond this difference in setup, we characterize when difficulty-based selection helps in this likelihood-based regime: the relative value of high- versus low-likelihood supervision depends jointly on model capacity and training duration. We further introduce a new axis, model capacity, into the theory of when difficulty-based selection helps, which prior work has mainly characterized along two axes: a data-size axis~\citep{sorscher2022beyond} and a learning-stage axis~\citep{weinshall2018curriculum}.}

%% file: sections/Problem_Setup.tex
\section{Background}
\label{sec:background}
Following a common practice in modern LLM development, smaller language models are often trained on supervision generated by stronger teacher models. For instance, AceReason-Nemotron 1.1~\citep{liu2026acereasonnemotron} is trained on data generated by DeepSeek-R1~\citep{guo2025deepseek}, while DASD-4B-Thinking~\citep{yan2026distribution} is supervised-fine-tuned from Qwen3-4B-Instruct-2507~\citep{yang2025qwen3technicalreport} using teacher responses produced by gpt-oss-120b~\citep{openai2025gptoss120bgptoss20bmodel}. Since the pool of teacher-generated responses can be very large, a central question is how to select the most useful supervision. GRAPE~\citep{zhang2025the} addresses this by favoring responses that are more aligned with the student’s distribution.

Let $\mathcal{Q}=\{x_i\}_{i=1}^N$ denote the instruction pool and let $\mathcal{T}$ denote the set of teacher models. For each teacher \(T \in \mathcal{T}\), let \(\pi_{\theta_T}(y \mid x)\) denote the conditional distribution over response \(y\) given instruction \(x\), where \(\theta_T\) is the parameter of teacher \(T\). For each instruction \(x_i \in \mathcal{Q}\), we collect one or more teacher-generated responses from every teacher \(T \in \mathcal{T}\). The resulting candidate response set is
\begin{equation}
    \mathcal{A}_i = \left\{ y_i^{(T,j)} \;:\; T \in \mathcal{T},\; j \in [J_T(i)] \right\},
    \qquad [J_T(i)] = \{1,\dots,J_T(i)\},
\end{equation}
where $J_T(i)$ is the number of responses generated by the teacher $T$ for instruction $x_i$.

The dataset $\mathcal{D}$ is obtained by selecting the answer $y_i^\star$ for each question $x_i$ with the highest log-likelihood under the base student model distribution $\pi_{\theta_0}(y \mid x_i)$, where $\theta_0$ is the parameters of the base student model,
\begin{equation}
    \mathcal{D} = \{(x_i, y_i^\star)\}_{i=1}^N,\qquad\qquad y_i^\star = \argmax_{y \in \mathcal{A}_i} \pi_{\theta_0}(y \mid x_i).
\end{equation}
The student is then trained with standard supervised fine-tuning:
\begin{equation}
    \min_{\theta} \; -\sum_{(x,y)\in \mathcal{D}} \log \pi_{\theta}(y \mid x).
\end{equation}

We examine this likelihood-based selection via two selection strategies: one formed from the highest-likelihood responses, $\mathcal{D}_{\textrm{high}}$, and one from the lowest-likelihood responses, $\mathcal{D}_{\textrm{low}}$,

\begin{align}
    \mathcal{D}_{\textrm{high}} = \{(x_i, y_i^\mathcal{H})\}_{i=1}^N,\qquad\qquad & y_i^\mathcal{H} = \argmax_{y \in R_i} \pi_{\theta_0}(y \mid x_i), \\
    \mathcal{D}_{\textrm{low}} = \{(x_i, y_i^\mathcal{L})\}_{i=1}^N,\qquad\qquad & y_i^\mathcal{L} = \argmin_{y \in R_i} \pi_{\theta_0}(y \mid x_i).
\end{align}

%% file: sections/Empirical_Results.tex
\section{Main Empirical Results}
\label{section:validate}
\subsection{Experimental Setup}
\textbf{Dataset \& Models.} We conduct our experiments on the MATH12K dataset~\citep{hendrycks2021measuring}. We use two teacher models, \texttt{Qwen/Qwen2.5-72B} \citep{yang2024qwen2} and \texttt{google/gemma-3-27b-it}~\citep{gemmateam2025gemma3technicalreport}, to generate answers for each question in MATH12K. We mainly report results of these student models, including \texttt{Qwen/Qwen2.5-1.5B}, \texttt{Qwen/Qwen2.5-3B-Instruct}, \texttt{Qwen/Qwen2.5-Math-7B}~\citep{qwen2025qwen25technicalreport}, \texttt{Qwen/Qwen3-4B}, \texttt{Qwen/Qwen3-8B}~\citep{yang2025qwen3technicalreport}. We show similar trends for other model families in \Appendixref{appendix:additional_results}, and provide training hyperparameters in \Appendixref{appendix:training}.

\textbf{Evaluation Metrics.} We evaluate correctness based on standard $\operatorname{pass@1}$~\citep{zhang2025the}. We set the temperature to 0.6, which is widely used in related works \citep{ye2025limo,wang2025reinforcement}. We evaluate on 10 popular math datasets: AIME24~\citep{aimo2024aime}, AMC~\citep{aimo2024amc}, CHMATH~\citep{wei2023cmathlanguagemodelpass}, Gaokao~\citep{zhang2023evaluating}, GPQA~\citep{rein2024gpqa}, GradeSchool~\citep{ye2025limo}, KAOYAN, MATH500~\citep{hendrycks2021measuring}, Minerva~\citep{NEURIPS2022_18abbeef}, Olympiad Bench~\citep{he2024olympiadbench}.

\subsection{Results}

\Tableref{tab:results_one_epoch} summarizes the best performance achieved within the first epoch for models of different sizes trained on high-likelihood and low-likelihood data. To capture early reasoning performance dynamics, we save four checkpoints during the first epoch and evaluate each of them; the reported result is the best score across these checkpoints. We observe a clear early advantage for high-likelihood selection: it wins on almost all datasets across model sizes, especially for the 1.5B and 3B students, on 9/10 for the 4B student, on 8/10 datasets for the 7B student with two ties, and on 9/10 datasets for the 8B student. This pattern suggests that high-likelihood data provides a stronger optimization signal in the early stage of training, especially for smaller-capacity students.

\definecolor{modelrow}{RGB}{221,235,247}
\definecolor{settingrow}{RGB}{242,242,242}
\definecolor{datasetcol}{RGB}{248,248,248}
\definecolor{highwin}{RGB}{255,235,156}
\newcommand{\legendbox}[1]{\textcolor{#1}{\rule{0.9em}{0.9em}}}

\newcommand{\highbetter}[1]{\cellcolor{highwin}\textbf{#1}}

\begin{table*}[t]
\centering
\scriptsize
\setlength{\tabcolsep}{5pt}
\renewcommand{\arraystretch}{1.12}
\caption{\textbf{Best performance across checkpoints after \textcolor{blue}{ONE} epoch for models of different sizes trained on high-likelihood and low-likelihood data}. Highlighted cells \legendbox{highwin} mark the better performance between training on high and low-likelihood data.}
\label{tab:results_one_epoch}
\resizebox{\textwidth}{!}{%
\begin{tabular}{>{\columncolor{datasetcol}}lcccccccccc}
\toprule
\rowcolor{modelrow}
& \multicolumn{2}{c}{\textbf{1.5B}} & \multicolumn{2}{c}{\textbf{3B}} & \multicolumn{2}{c}{\textbf{4B}} & \multicolumn{2}{c}{\textbf{7B}} & \multicolumn{2}{c}{\textbf{8B}} \\
\cmidrule(lr){2-3}\cmidrule(lr){4-5}\cmidrule(lr){6-7}\cmidrule(lr){8-9}\cmidrule(lr){10-11}
\rowcolor{settingrow}
\textbf{Dataset} & \textbf{High} & \textbf{Low} & \textbf{High} & \textbf{Low} & \textbf{High} & \textbf{Low} & \textbf{High} & \textbf{Low} & \textbf{High} & \textbf{Low} \\
\midrule
\textbf{AIME}             & \highbetter{3.33}  & 0     & \highbetter{10.00} & 3.33  & \highbetter{20.00} & 10.00 & \highbetter{13.33} & \highbetter{13.33} & \highbetter{36.67} & 26.67 \\
\textbf{AMC}              & \highbetter{27.50} & 17.50 & \highbetter{47.50} & 32.50 & \highbetter{62.00} & 50.00 & \highbetter{62.50} & 47.50 & \highbetter{87.50} & 80.00 \\
\textbf{CHMATH}           & \highbetter{6.67}  & 3.33  & \highbetter{10.00} & 3.33  & \highbetter{26.67} & 13.33 & \highbetter{30.00} & 20.00 & \highbetter{56.67} & 50.00 \\
\textbf{GAOKAO}           & \highbetter{10.13} & 6.33  & \highbetter{25.32} & 10.13 & \highbetter{54.43} & 53.16 & \highbetter{55.70} & \highbetter{55.70} & \highbetter{74.68} & 70.89 \\
\textbf{GPQA}             & \highbetter{28.79} & 16.67 & \highbetter{34.85} & 18.18 & \highbetter{36.36} & 26.26 & \highbetter{39.39} & 27.78 & \highbetter{78.79} & 75.25 \\
\textbf{GRADE SCHOOL}     & \highbetter{22.38} & 13.33 & \highbetter{19.52} & 14.76 & 48.57 & \highbetter{40.95} & \highbetter{48.10} & 44.29 & \highbetter{68.10} & 65.24 \\
\textbf{KAOYAN}           & \highbetter{19.10} & 8.04  & \highbetter{24.12} & {7.04} & \highbetter{40.20} & 25.13 & \highbetter{46.23} & 31.16 & \highbetter{71.86} & 65.33 \\
\textbf{MATH}             & \highbetter{54.60} & 41.80 & \highbetter{66.40} & 52.20 & \highbetter{75.20} & 66.00 & \highbetter{78.80} & 71.80 & 89.20 & \highbetter{91.60} \\
\textbf{MINERVA}          & \highbetter{24.06} & 13.24 & \highbetter{32.35} & 25.37 & \highbetter{40.44} & 33.82 & \highbetter{43.01} & 33.82 & \highbetter{58.46} & 57.72 \\
\textbf{OLYMPIAD BENCH}   & \highbetter{21.78} & 13.04 & \highbetter{26.52} & 20.59 & \highbetter{37.48} & 31.70 & \highbetter{42.37} & 37.93 & \highbetter{64.30} & 64.00 \\
\bottomrule
\end{tabular}%
}
\end{table*}

\begin{table*}[t]
\centering
\small
\setlength{\tabcolsep}{5pt}
\renewcommand{\arraystretch}{1.12}
\caption{\textbf{Best performance across checkpoints after \textcolor{blue}{FIVE} epochs for models of different sizes trained on high-likelihood and low-likelihood data}. Highlighted cells \legendbox{highwin} mark the better performance between training on high and low-likelihood data.}
\label{tab:results_five_epochs}
\resizebox{\textwidth}{!}{%
\begin{tabular}{>{\columncolor{datasetcol}}lcccccccccc}
\toprule
\rowcolor{modelrow}
& \multicolumn{2}{c}{\textbf{1.5B}} & \multicolumn{2}{c}{\textbf{3B}} & \multicolumn{2}{c}{\textbf{4B}} & \multicolumn{2}{c}{\textbf{7B}} & \multicolumn{2}{c}{\textbf{8B}} \\
\cmidrule(lr){2-3}\cmidrule(lr){4-5}\cmidrule(lr){6-7}\cmidrule(lr){8-9}\cmidrule(lr){10-11}
\rowcolor{settingrow}
\textbf{Dataset} & \textbf{High} & \textbf{Low} & \textbf{High} & \textbf{Low} & \textbf{High} & \textbf{Low} & \textbf{High} & \textbf{Low} & \textbf{High} & \textbf{Low} \\
\midrule
\textbf{AIME}             & \highbetter{6.67}  & 3.33  & \highbetter{10.00} & 6.67  & \highbetter{20.00} & \highbetter{20.00} & 16.67 & \highbetter{26.67} & 40.00 & \highbetter{43.33} \\
\textbf{AMC}              & \highbetter{35.00} & 22.50 & \highbetter{55.00} & 47.50 & \highbetter{70.00} & 57.50 & \highbetter{67.50} & 62.50 & 90.00 & \highbetter{92.5} \\
\textbf{CHMATH}           & \highbetter{6.67}  & 6.67  & \highbetter{16.67} & 10.00 & 33.33 & \highbetter{36.67} & \highbetter{33.33} & \highbetter{33.33} & 56.67 & \highbetter{63.33} \\
\textbf{GAOKAO}           & \highbetter{11.39} & 8.86  & \highbetter{25.32} & 24.05 & 56.96 & \highbetter{58.23} & 55.70 & \highbetter{62.03} & \highbetter{74.68} & 67.09 \\
\textbf{GPQA}             & \highbetter{29.29} & 16.67 & \highbetter{34.85} & 27.27 & \highbetter{39.39} & 38.38 & 39.39 & \highbetter{44.95} & 78.79 & \highbetter{79.8} \\
\textbf{GRADE SCHOOL}     & \highbetter{22.86} & 13.33 & \highbetter{22.38} & 19.52 & \highbetter{53.33} & 47.14 & 50.95 & \highbetter{55.71} & 70.48 & \highbetter{72.38} \\
\textbf{KAOYAN}           & \highbetter{19.60} & 10.55 & \highbetter{26.63} & {11.56} & \highbetter{44.22} & 33.67 & \highbetter{46.73} & 37.69 & 69.85 & \highbetter{71.86} \\
\textbf{MATH}             & \highbetter{57.60} & 44.40 & \highbetter{66.80} & 58.60 & \highbetter{78.8} & 74.4 & 79.00 & \highbetter{80.06} & 92.00 & \highbetter{93.00} \\
\textbf{MINERVA}          & \highbetter{22.79} & 14.71 & \highbetter{32.72} & 25.74 & 40.81 & \highbetter{45.96} & \highbetter{43.38} & 37.87 & 58.46 & \highbetter{62.50} \\
\textbf{OLYMPIAD BENCH}   & \highbetter{22.52} & 14.96 & \highbetter{30.52} & 24.30 & \highbetter{41.63} & 41.19 & 42.81 & \highbetter{44.3} & 65.19 & \highbetter{65.48} \\
\bottomrule
\end{tabular}%
}
\end{table*}

\begin{observation}[Early-stage advantage of high-likelihood data]
    High-likelihood data consistently yield faster early gains, especially for smaller models.
\end{observation}

\Tableref{tab:results_five_epochs} reports the best performance over the full 5-epoch training process. During training, we save four checkpoints per epoch and evaluate all of them, so each entry corresponds to the best result among 20 checkpoints. Unlike the first-epoch results, where high-likelihood data consistently provides a stronger early learning signal, the longer training horizon reveals a clear capacity-dependent reversal. For the smaller 1.5B, 3B, and 4B students, high-likelihood training remains dominant. However, for the larger 7B and 8B students, low-likelihood training surpasses high-likelihood training on most tasks.

\begin{observation}[Low-likelihood data helps large models eventually]
    For high-capacity models, high-likelihood peaks early, while low-likelihood improves later and can surpass it.
\end{observation}

\definecolor{headerblue}{RGB}{221,235,247}
\definecolor{lightgray}{RGB}{242,242,242}

\begin{table*}[t]
\centering
\small
\setlength{\tabcolsep}{10pt}
\renewcommand{\arraystretch}{1.15}
\caption{\textbf{An illustrative example showing that larger models are less prone to getting lost under low-likelihood training.} The 1.5B model repeats the question and then gets stuck repeating an incorrect final answer, the 3B model repeats intermediate reasoning, and the 7B model successfully completes the reasoning and returns the correct answer.}
\label{tab:capacity_get_lost_example}
\resizebox{\textwidth}{!}{%
\begin{tabular}{>{\centering\arraybackslash}p{0.31\textwidth} >{\centering\arraybackslash}p{0.31\textwidth} >{\centering\arraybackslash}p{0.31\textwidth}}
\toprule
\rowcolor{headerblue}
\multicolumn{3}{c}{\textbf{Question:} There is a sequence $a_1,a_2,\ldots,a_{2023}$ such that $\tan(2023x)=\cdots$. What is $a_{2023}$?} \\
\midrule
\rowcolor{lightgray}
\textbf{Qwen2.5-1.5B} & \textbf{Qwen2.5-3B-Instruct} & \textbf{Qwen2.5-Math-7B} \\
\midrule
There is a sequence $a_1,a_2,\ldots,a_{2023}$ such that $\tan(2023x)=\cdots$. What is $a_{2023}$?
\[
{{\color{red}\boxed{2023}\ \boxed{2023}\ \boxed{2023}\ \boxed{2023}}\ \cdots}
\]
&
Let $P(x)=\tan 2023x$. Then
\[
\tan(2023x)=\frac{\tan(2022x)+\tan x}{1-\tan(2022x)\tan x}
\]
\[
{\color{red}\tan(2023x)=\frac{\tan(2022x)+\tan x}{1-\tan(2022x)\tan x}}
\]
\[
\cdots
\]
&
Let $t = \tan x$. Then we have...\newline
We know that
\[
\tan (A+B) = \frac{\tan A + \tan B}{1 - \tan A \tan B}
\]
\[
\cdots
\]
{\color{pastelgreen}Final answer: \(\boxed{1}\).} \rule{0pt}{0.9em} \\[-1.7em]
\bottomrule
\end{tabular}%
}
\end{table*}

We present an illustrative example to highlight how model capacity affects learning from low-likelihood data in \Tableref{tab:capacity_get_lost_example}, showing that the 1.5B model mostly repeats the question and then produces an incorrect final answer, indicating shallow imitation rather than acquisition of the teacher’s reasoning process. The 3B model can generate some preliminary reasoning steps, but then gets lost in repetition and still fails to complete the solution. By contrast, the 7B model successfully carries out the reasoning and reaches the correct answer.

\begin{observation}[Small model learns by rote]
    Small model: repetition, shallow copying, incomplete reasoning. Larger model: delayed but meaningful reasoning improvement.
\end{observation}

In brief, these results reveal a clear ``{\color{SMALLCOLOR}\textbf{Fast-Fit}} / {\color{LARGECOLOR}\textbf{Slow-Gain}}'' pattern. In the {\color{SMALLCOLOR}\textbf{Fast-Fit}} regime: the model quickly aligns to clean, learnable supervision and improves with relatively little training. The effect is especially strong for smaller students, whose limited capacity makes them more dependent on examples that match their current representation and optimization ability. By contrast, in the {\color{LARGECOLOR}\textbf{Slow-Gain}} regime, low-likelihood data often appears less useful at the beginning, but becomes increasingly valuable when the student has more capacity and the training budget allows longer optimization. This suggests that the best data-selection strategy is not universal: it should be chosen jointly according to how powerful the student is and how much training compute is available.

\begin{takeaway}[Computing budget and model capacity matter data selection]
    High-likelihood data is best for fast, compute-efficient early learning, whereas low-likelihood data becomes more valuable when model capacity and training budget are large enough.
\end{takeaway}

%% file: sections/Deep_Analysis.tex
\section{\raisebox{-0.2em}{\includegraphics[height=1.5em]{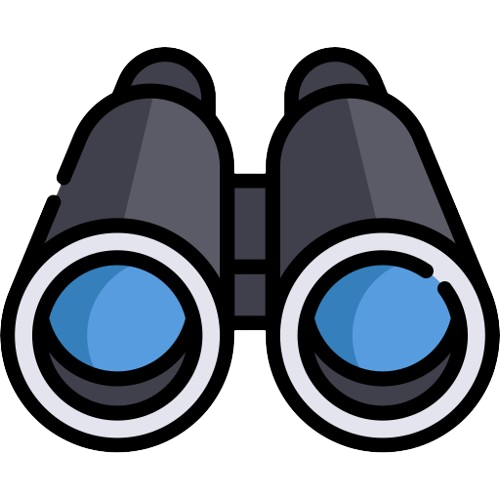}}\hspace{0.4em}Dual-view Learning Dynamics Analysis}
\label{sec:learning_dynamics_analysis}

In this section, we further examine why the {\color{LARGECOLOR}\textbf{Slow-Gain}} regime appears for larger models, yet disappears for small-capacity students from teacher and original model views.

\begin{figure}[ht]
    \centering
    \includegraphics[width=\linewidth]{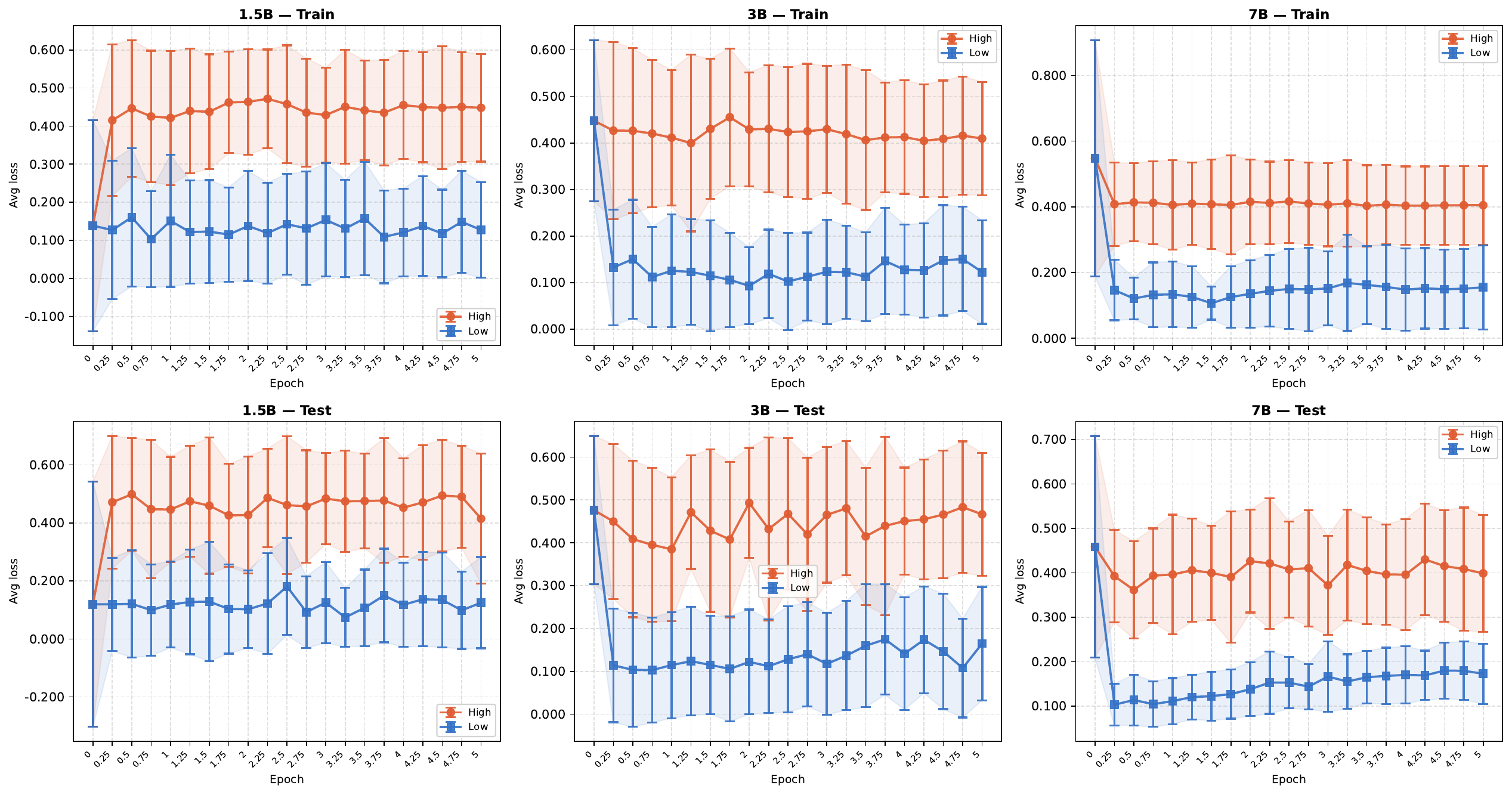}
    \caption{\textbf{Loss dynamics under the teacher model.} We measure the Gemma log-likelihood of answers on the test set (AMC) and training set (MATH12K) across checkpoints for the 1.5B, 3B, and 7B student models trained on low- and high-likelihood data.}
    \label{fig:loss_teacher}
\end{figure}

\begin{figure}[ht]
    \centering
    \includegraphics[width=\linewidth]{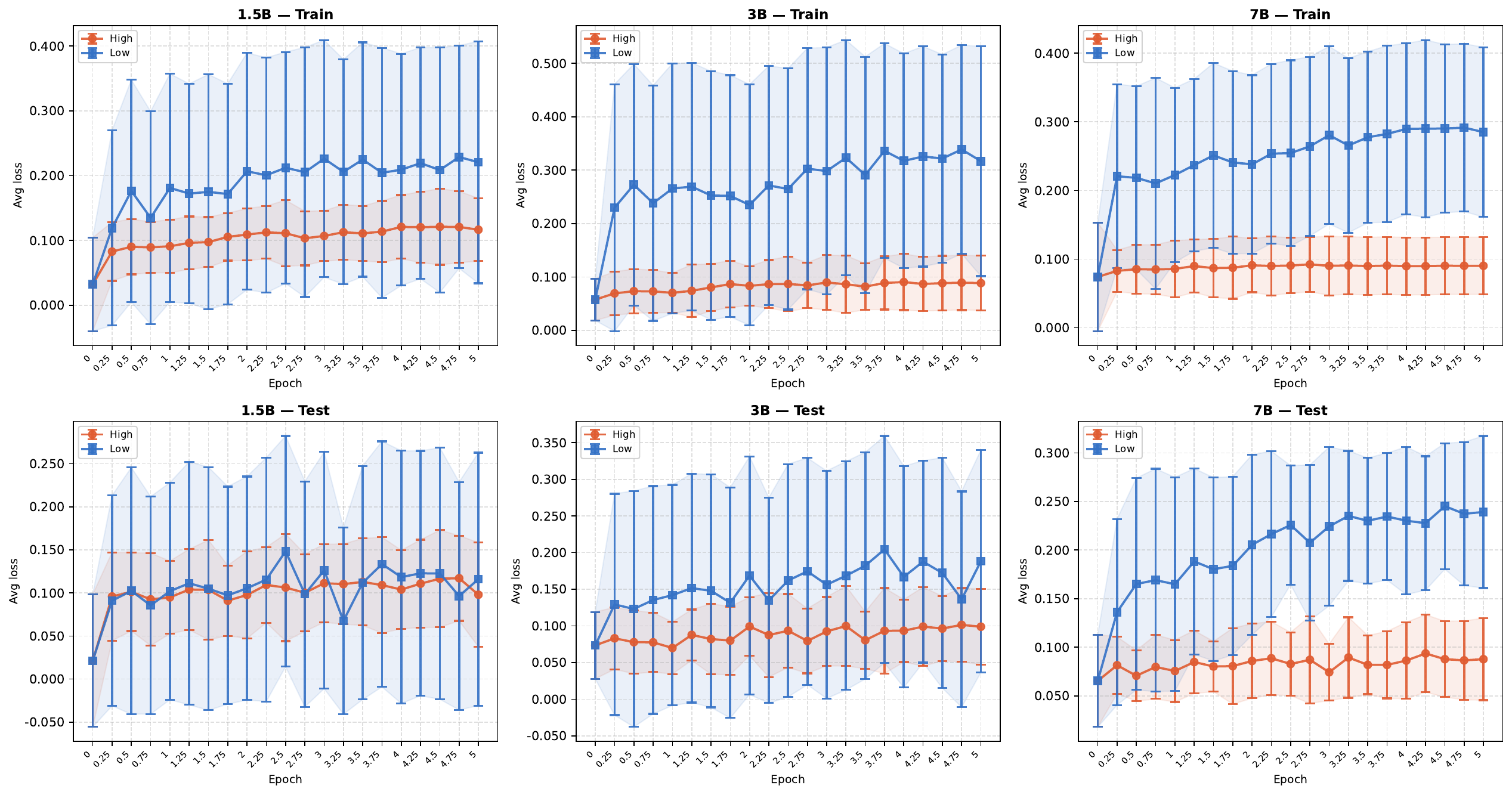}
    \caption{\textbf{Loss dynamics under the base model.} Log-likelihood on AMC (test) and MATH12K (train) across checkpoints for 1.5B, 3B, and 7B students, evaluated under their own base models after training on low- and high-likelihood data.}
    \label{fig:loss_base}
\end{figure}

\textbf{Learning Dynamics under Teacher View.} \Figureref{fig:loss_teacher} reveals a clear capacity-dependent teacher-loss dynamics pattern. For the 1.5B student, training on high-likelihood data causes the teacher-model loss of the student’s generated answers to increase, indicating that the student is drifting farther away from the teacher distribution. When trained on low-likelihood data, the teacher loss remains roughly unchanged. Taken together, these results suggest that the 1.5B model lacks sufficient capacity to move meaningfully closer to the teacher, regardless of whether the selected data are high- or low-likelihood. For the 3B student, the pattern becomes more differentiated: under high-likelihood training, the teacher loss stays largely stable, whereas under low-likelihood training it decreases, showing that low-likelihood data helps this model move closer to the teacher distribution. For the 7B student, the teacher loss decreases under both high- and low-likelihood training, indicating that, once model capacity is sufficiently large, the student can absorb teacher knowledge effectively regardless of the likelihood-based data selection. Overall, these results suggest that the ability of likelihood-based data selection to pull the student toward the teacher depends strongly on model capacity: smaller students struggle to approach the teacher at all, medium-sized students benefit particularly from low-likelihood data, and larger students can align with the teacher under either selection strategy.

\begin{observation}[Capacity-Dependent Teacher Alignment]
    Low-likelihood data moves larger models closer to the teacher but fails to do so for small models.
\end{observation}







\textbf{Learning Dynamics under Original Model View.} \Figureref{fig:loss_base} presents the loss of generated answers measured under the base model, which helps reveal how far the fine-tuned student departs from its original distribution. For the 1.5B model, training on low-likelihood data leads to a broader and more divergent loss pattern under the base model, while still overlapping substantially with the initial loss range. This suggests that although fine-tuning pushes many generated answers away from the base distribution, a considerable portion of them remain close to the model’s original behavior. For the 3B model, answers produced after training on low-likelihood data exhibit noticeably higher loss under the base model, indicating a stronger shift away from the original distribution than in the 1.5B case. Moreover, the separation between the answers generated by models trained on high-likelihood versus low-likelihood data becomes larger, showing that likelihood-based data selection has a clearer and more distinct effect at this scale. For the 7B model, the loss variance of answers generated after low-likelihood training becomes smaller than for the 3B model, suggesting that the larger model adapts in a more stable and focused way rather than drifting broadly. On the test set, the difference between the 1.5B models trained on high- and low-likelihood data is relatively small, but this gap widens as model size increases. This indicates that larger models not only generalize better, but also express the effect of data selection more clearly in their output distribution.

\begin{observation}[Capacity-Dependent Distribution Shift and Generalization]
    As model size increases, fine-tuned behavior becomes more stable, generalization improves, and the distinction between high-likelihood and low-likelihood training becomes more pronounced.
\end{observation}

%% file: sections/Theory.tex
\section{Theoretical Mechanism of Capacity-Constrained Distillation}
\label{sec:theory}

\definecolor{STUDENTCOLOR}{HTML}{CC00CC}

\begin{figure}[ht]
    \centering
    \includegraphics[width=.6\linewidth]{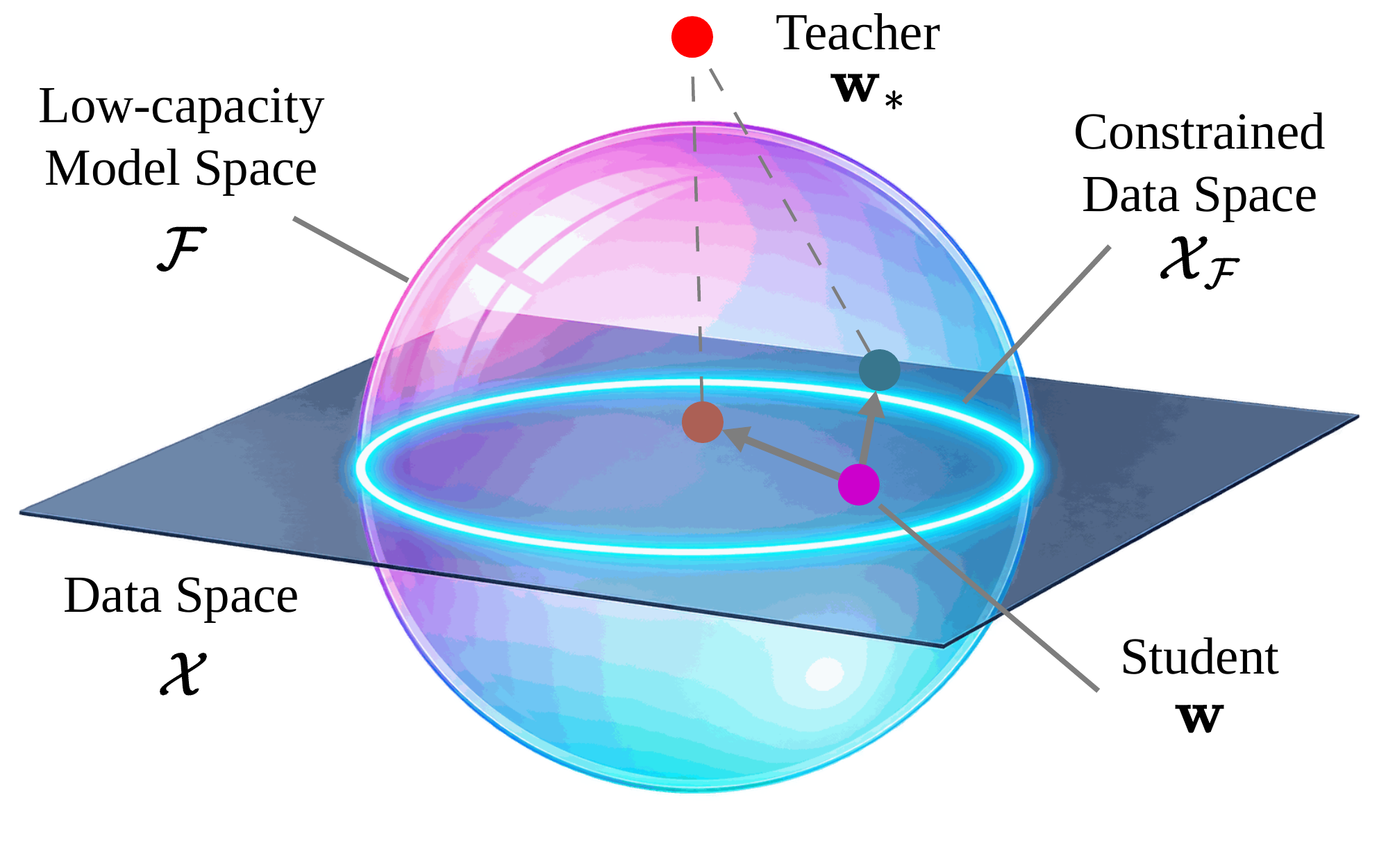}
    \caption{\textbf{Geometric intuition for distillation with a low-capacity student.} The data space acts as a mirror that reveals teacher knowledge to the student, but limited capacity restricts learning to the portion that lies in the student’s feasible space, thereby shrinking the effective supervision. As a result, high-capacity students {\color{STUDENTCOLOR}\raisebox{-0.4ex}{\scalebox{2}{$\bullet$}}} can absorb more of the exposed knowledge and move toward {\color{LARGECOLOR}\raisebox{-0.4ex}{\scalebox{2}{$\bullet$}}}, whereas low-capacity students move only toward {\color{SMALLCOLOR}\raisebox{-0.4ex}{\scalebox{2}{$\bullet$}}}.}
    \label{fig:distill_sol}
\end{figure}

In this section, we provide a theoretical account of the capacity-dependent effects observed in our experiments. Our goal is to clarify why the same data selection can produce markedly different outcomes for students of different capacities, and in particular, why low-likelihood data may benefit larger models while failing to help smaller ones.

\textbf{Setting.} We study knowledge distillation for binary classification, where a student model learns from labels generated by a teacher model. Let the teacher and student be linear classifiers \(h^\ast, h : \mathcal{X} \to \mathcal{Y}\), defined by
\begin{equation}
    h^\ast(\mathbf{x}) = \mathbf{1}\{\mathbf{w}_\ast^\top \mathbf{x} \ge 0\},\qquad
    h(\mathbf{x}) = \mathbf{1}\{\mathbf{w}^\top \mathbf{x} \ge 0\},
\end{equation}
where \(\mathbf{w}_\ast, \mathbf{w} \in \mathbb{R}^d \setminus \{0\}\) are the teacher and student parameter vectors, respectively, and \(\mathbf{1}\{\cdot\}\) denotes the indicator function. The weight vectors can also be parameterized as the product of matrices, $\mathbf{w}^\top=\mathbf{W}_N\mathbf{W}_{N-1}\cdots\mathbf{W}_{1}$ for some $N\ge1$. When $n\ge2$, this parameterization is known as \textit{deep linear network}, whose analysis is deferred in \Appendixref{appendix:deep_linear_network}. For a low-capacity student, we assume that the feasible student parameter space is a subspace \(\mathcal{F} \subseteq \mathbb{R}^d\), and let $P_{\mathcal{F}} : \mathbb{R}^d \to \mathcal{F}$ denote the orthogonal projection onto \(\mathcal{F}\). The student is trained on a dataset $\mathcal{D} = \{(\mathbf{x}_i, y_i)\}_{i=1}^n$, where each soft label $y_i$ is provided by the teacher,
\begin{equation}
    y_i = \sigma(\mathbf{w}_\ast^\top \mathbf{x}_i),
\end{equation}
with \(\sigma(\cdot)\) denoting the sigmoid function. Let \(\mathbf{w}(\tau)\) denote the student parameter at training time \(\tau\). The student is updated by gradient flow on the normalized cross-entropy loss \(L^1\).

\begin{equation}
    L^1(\mathbf{w}) =-\frac{1}{n}\sum_{i=1}^n\left[ y_i\log\sigma\left(\mathbf{w}^\top\mathbf{x}_i\right)+\left(1-y_i\right)\log\left(1-\sigma\left(\mathbf{w}^\top\mathbf{x}_i\right)\right) \right] - L^\ast,
\end{equation}

where $L^\ast$ is a normalization constant. Detailed setting description is in \Appendixref{appendix:setting}.

\begin{theorem}[High-capacity Student Solution]
\label{theorem:high-cap_student_sol}
Let \(P_{\mathbf X}\) be the orthogonal projector onto \(\operatorname{span}\{\mathbf{x}_1,\dots,\mathbf{x}_n\}\). Then, as \(\tau \to \infty\), the student parameter, $\mathbf{w}(\tau)$, converges to $\hat{\mathbf{w}}$, where

\begin{equation}
    \hat{\mathbf{w}} = {\color{pastelpurple}\mathbf{w}_0} + {\color{pastelred}P_{\mathbf X}}({\color{pastelgreen}\mathbf{w}_\ast-\mathbf{w}_0}).
\end{equation}
\end{theorem}

\begin{proof}
    The proof is deferred in \Appendixref{appendix:one_layer_network}.
\end{proof}

\textbf{Interpretation.} \Theoremref{theorem:high-cap_student_sol} shows that training can only transfer teacher knowledge that is exposed by the data. The vector {\color{pastelgreen}\(\mathbf{w}_\ast-\mathbf{w}_0\)} represents the knowledge gap between teacher and student. The projection term \({\color{pastelred}P_{\mathbf X}}({\color{pastelgreen}\mathbf{w}_\ast-\mathbf{w}_0})\) indicates that the knowledge gap between the teacher and the student is transferred only through its component in the data span. Thus, the data act as a mirror of the teacher: if the data span aligns well with the knowledge gap, the student learns effectively; if it is orthogonal to that gap, then no transfer occurs. Moreover, directions orthogonal to the data span are never updated, so the final student remains partially constrained by its initialization {\color{pastelpurple}\(\mathbf{w}_0\)}.

\begin{theorem}[Low-capacity student solution]
\label{theorem:low-cap_student_sol}
Let \(\mathbf{X}_{\mathcal{F}}\) be the projection of the data span onto \(\mathcal{F}\), and let \(P_{\mathbf{X}_{\mathcal{F}}}\) denote the orthogonal projection onto \(\mathbf{X}_{\mathcal{F}}\). Then, as \(\tau \to \infty\), the student parameter converges to $\hat{\mathbf{w}}$, where
\begin{equation}
    \hat{\mathbf{w}} = \mathbf{w}_0 + {\color{pastelorange}P_{\mathbf{X}_{\mathcal{F}}}}(\mathbf{w}_\ast-\mathbf{w}_0).
\end{equation}
\end{theorem}
\begin{proof}
    The proof is deferred in \Appendixref{appendix:one_layer_network}.
\end{proof}

\definecolor{SHARED_SPACE}{HTML}{BFFBFB}

\textbf{Interpretation.} The projection term ${\color{pastelorange}P_{\mathbf{X}_{\mathcal{F}}}}$ in \Theoremref{theorem:low-cap_student_sol} shows that, for a low-capacity student, only the portion of the data span that lies in the feasible student space \(\mathcal{F}\) is useful for learning, as shown in \Figureref{fig:distill_sol}. Hence, distillation is constrained not only by what the data reveal about the teacher, but also by what the student can represent. 

\begin{takeaway}[Capacity-aligned data is necessary for low-capacity distillation]
    For low-capacity students, successful distillation requires not only data that exposes teacher knowledge, but data whose exposed knowledge lies within the student’s representable subspace.
\end{takeaway}

\textbf{Connection to Failure of Low-likelihood Data for Small LLMs.} 
Low-likelihood data can be useful because it exposes teacher knowledge beyond the student’s current distribution. However, our theory shows that a low-capacity student can learn only the portion of this knowledge that lies in its feasible subspace. When the likelihood is too low, the exposed teacher knowledge may exceed what the student can represent or optimize, making the supervision ineffective.

%% file: sections/Conclusion.tex
\section{Conclusion}

This work challenges the prevailing assumption that high-likelihood data is universally optimal for fine-tuning LLMs on reasoning tasks. Through controlled experiments across model scales (1.5B–8B parameters) and analysis of learning dynamics, we demonstrate a clear capacity-dependent ``{\color{SMALLCOLOR}\textbf{Fast-Fit}} / {\color{LARGECOLOR}\textbf{Slow-Gain}}'' pattern: high-likelihood data yields faster and more stable early improvements, particularly for smaller models, while low-likelihood data becomes increasingly beneficial for larger models given sufficient training budget. 

%% file: sections/Acknowledgement.tex
\section*{Acknowledgements}
Ruoxi Jia and the ReDS lab acknowledge support from the National Science Foundation through grants IIS-2312794, IIS-2313130, OAC-2239622, CNS-2424127, OAC-2613761, and the Amazon-Virginia Tech Initiative for Efficient and Robust Machine Learning. We also acknowledge Advanced Research Computing at Virginia Tech for providing computational resources and technical support that have contributed to the results reported within this paper.

%% file: sections/Appendix.tex
\clearpage

\appendix

\section{Missing Proofs}

\subsection{Setting}
\label{appendix:setting}
\textbf{Teacher-Student Model Notations.} We formally introduce distillation in the context of binary
classification. Let $\mathcal{X}$ be the input space, $\mathcal{Y}=\left\{0,1\right\}$ be the label space, and $\mathcal{D}$ the probability distribution of inputs.

The teacher $h^\ast:\mathcal{X}\rightarrow\mathcal{Y}$ is a fixed linear classifier, i.e.

\begin{equation}
    h^\ast\left(\mathbf{x}\right) = \mathbf{1}\left\{\mathbf{w}_\ast^\top \mathbf{x} \ge 0\right\},
\end{equation}
for some $\mathbf{w}_\ast\in\mathbb{R}^d\backslash\left\{0\right\}$, where $\mathbf{1}\left\{.\right\}$ returns 1 if the argument is true and 0 otherwise. The student is also a linear classifier, i.e.
\begin{equation}
    h\left(\mathbf{x}\right)=\mathbf{1}\left\{\mathbf{w}^\top\mathbf{x}\geq0\right\}.
\end{equation}

\textbf{Low-capacity students lie in a subspace of teacher models.} We model the low-capacity student class as an $s$-dimensional subspace $\mathcal S \subset \mathbb{R}^d$, and let $P_{\mathcal S}$ denote the orthogonal projection from the teacher parameter space $\mathbb{R}^d$ onto $\mathcal S$. The orthogonal projection assumption is realistic for many teacher-student pairs used in practice. In model compression and distillation, the student often preserves the teacher's architectural form while reducing capacity, for example, by using fewer transformer layers, fewer attention heads, or smaller hidden dimensions \citep{jiao2020tinybert,muralidharan2024compact,yan2026distribution}. Under this view, the student parameterization can be interpreted as a lower-dimensional subspace of the teacher parameter space, and the representable component of teacher knowledge is naturally modeled as the orthogonal projection onto this subspace.

\textbf{Distillation Dataset Construction.} Distillation proceeds as follows. We collect a transfer
set $\left\{\left(\mathbf{x}_i,y_i\right)\right\}_{i=1}^n$ consisting of inputs $\mathbf{x}_i$ sampled i.i.d. from $\mathcal{D}$ and soft labels
\begin{equation}
    y_i=\sigma\left(\mathbf{w}_\ast^\top\mathbf{x}_i\right)
\end{equation}
provided by the teacher, where $\sigma$ is the sigmoid function,
\begin{equation}
    \sigma\left(x\right)=\frac{1}{1+\operatorname{exp}\left(-x\right)}\;.
\end{equation}
The soft real-valued labels can be thought of as a more informative version of the hard (0/1-valued) labels of the standard classification setting. We denote
\begin{equation}
    \mathbf{X}=\left[\mathbf{x}_1,\dots,\mathbf{x}_n\right]\in\mathbb{R}^{d\times n}
\end{equation}
the data matrix.

\textbf{Student Model Optimization.} The student is trained by
minimizing the (normalized) cross-entropy loss
\begin{align}
\label{eq:cross_entropy}
   L^1(\mathbf{w})&=-\frac{1}{n}\sum_{i=1}^n \ell_i(\mathbf{w}^\top \mathbf{x}_i) - L^\ast \\
   &=-\frac{1}{n}\sum_{i=1}^n\left[ y_i\log\sigma\left(\mathbf{w}^\top\mathbf{x}_i\right)+\left(1-y_i\right)\log\left(1-\sigma\left(\mathbf{w}^\top\mathbf{x}_i\right)\right) \right] - L^\ast,
\end{align}
where $L^\ast$ is a normalization constant such that the minimum of $L^1$ is 0.

\Equationref{eq:cross_entropy} is just the simplified notation. For better analysis on optimization, especially deep neural network optimization, involving many parameter components, we define

\begin{equation}
    L\left(\mathbf{W}_1,\dots,\mathbf{W}_N\right):= L^1\left(\left(\mathbf{W}_N\mathbf{W}_{N-1}\cdots\mathbf{W}_1\right)^\top\right),
\end{equation}
and optimize it via gradient descent. We write $\mathbf{W}_i\left(\tau\right)$ for the value of the matrix $\mathbf{W}_i$ at time $\tau\in[0,\infty)$, with $\mathbf{W}_i\left(0\right)$ denoting the initial value, and $\mathbf{w}\left(\tau\right)=\mathbf{W}_N\left(\tau\right)\cdots\mathbf{W}_1\left(\tau\right)$. Then, each $\mathbf{W}_i\left(\tau\right)$, for $i\in\{1,\dots,N\}$, evolves according to the forllowing differential equation,

\begin{equation}
    \frac{\partial\mathbf{W}_i\left(\tau\right)}{\partial\tau}=-\frac{\partial L}{\partial\mathbf{W}_i}\left(\mathbf{W}_1\left(\tau\right),\dots,\mathbf{W}_N\left(\tau\right)\right).
\end{equation}

\subsection{Properties of Cross-Entropy Loss}

In this section, we show below the properties of cross-entropy loss shown in \Equationref{eq:cross_entropy}.

\begin{proposition}[Space of Gradient Updates]
\label{theo:grad_constraint}
The gradient update is constrained in the data span, $\nabla L^1\left(\mathbf{w}\right)\in\operatorname{span}\left(\mathbf{X}\right)$.
\end{proposition}

\begin{proof}
    The gradient with respect to the student weight vector $\mathbf{w}$ is
\begin{equation}
    \nabla L^1(\mathbf{w})=\frac{1}{n}\sum_{i=1}^n \bigl(\sigma(\mathbf{w}^\top \mathbf{x}_i)-y_i\bigr)\cdot\mathbf{x}_i\,.
\end{equation}
Therefore, the gradient is always a linear combination of the training inputs ($\mathbf{x}_1,\dots,\mathbf{x}_n$), and hence

\begin{equation}
    \nabla L^1(\mathbf{w})\in \operatorname{span}(\mathbf{X})\,.    
\end{equation}
\end{proof}

\begin{proposition}[Global Minima Condition]
    The global minimum of the cross-entropy loss is 0, and the set of global minimisers is
    \begin{equation}
    \label{eq:optimal_condition}
        \left\{\mathbf{w}\in\mathbb{R}^d:\mathbf{X}^\top\mathbf{w}=\mathbf{X}^\top\mathbf{w}_{\ast}\right\}\,.
    \end{equation}
\end{proposition}

\begin{proof}
    We know that $L^1\geq 0$ and $L^1(\mathbf{w}_{\ast})=0$, so $0$ is the optimal objective value, and the set of global optima consists of all $\mathbf{w}$ such that $L^1\left(\mathbf{w}\right)=0$. The last condition is equivalent to $\forall_i:\ell_i\left(\mathbf{w}\right)=0$, which in turn is equivalent to $\forall_i:\sigma\left(\mathbf{w}^\top\mathbf{x}_i\right)=\sigma\left(\mathbf{w}_{\ast}^\top\mathbf{x}_i\right)$. By monotonicity of $\sigma$, this is further equivalent to $\forall_i:\mathbf{w}^\top\mathbf{x}_i=\mathbf{w}_{\ast}^\top\mathbf{x}_i$, which is a restatement of \Equationref{eq:optimal_condition}.
\end{proof}

\begin{lemma}[Bounds for Rayleigh Quotient]
\label{lem:rayleigh}
    Let $\mathbf{A}\in\mathbb{R}^{n\times n}$ be a real symmetric matrix, and let $\lambda_{\min}\left(\mathbf{A}\right)$ and $\lambda_{\max}\left(\mathbf{A}\right)$ denote its smallest and largest eigenvalues, respectively. Then
    \begin{equation}
        \lambda_{\min}\left(\mathbf{A}\right)\preceq\mathbf{A}\preceq\lambda_{\max}\left(\mathbf{A}\right),
    \end{equation}
    or equivalently, for every $\mathbf{z}\in\mathbb{R}^n$,
    \begin{equation}
        \lambda_{\min}\left(\mathbf{A}\right)\|\mathbf{z}\|^2\leq\mathbf{z}^\top\mathbf{A}\mathbf{z}\leq\lambda_{\max}\left(\mathbf{A}\right)\|\mathbf{z}\|^2,
    \end{equation}
\end{lemma}

\begin{proof}
    Since $\mathbf{A}$ is real symmetric, the spectral theorem gives an orthogonal matrix $\mathbf{Q}\in\mathbb{R}^{n\times n}$ and a diagonal matrix
    \begin{equation}
        \Lambda=\operatorname{diag}\left(\lambda_1,\dots,\lambda_n\right)
    \end{equation}
    such that
    \begin{equation}
        \mathbf{A}=\mathbf{Q}\Lambda\mathbf{Q}^\top,
    \end{equation}
    where $\lambda_1,\dots,\lambda_n$ are the eigenvalues of $\mathbf{A}$. Let
    \begin{equation}
        \lambda_{\min}\left(\mathbf{A}\right)=\min_{1\leq i\leq n}\,\lambda_i,\qquad\lambda_{\max}\left(\mathbf{A}\right)=\max_{1\leq i\leq n}\,\lambda_i.
    \end{equation}
    Define
    \begin{equation}
        \mathbf{y}:=\mathbf{Q}^\top\mathbf{z}.
    \end{equation}
    Because $\mathbf{Q}$ is orthogonal, it preserves Euclidean norm, so
    \begin{equation}
        \|\mathbf{y}\|=\|\mathbf{z}\|.
    \end{equation}
    Then
    \begin{equation}
    \mathbf{z}^\top\mathbf{A}\mathbf{z}=\mathbf{z}^\top\mathbf{Q}\Lambda\mathbf{Q}^\top\mathbf{z}=\mathbf{y}^\top\Lambda\mathbf{y}=\sum_{i=1}^n\lambda_i\mathbf{y_i^2}.
    \end{equation}
    Since $\lambda_i\leq\lambda_{\max}\left(\mathbf{A}\right)$ for all $i$, we have
    \begin{equation}
        \sum_{i=1}^n\lambda_{\min}\left(\mathbf{A}\right)\mathbf{y}_i^2\leq\sum_{i=1}^n\lambda_i\mathbf{y}_i^2\leq\sum_{i=1}^n\lambda_{\max}\left(\mathbf{A}\right)\mathbf{y}_i^2
    \end{equation}
    or
    \begin{equation}
        \lambda_{\min}\left(\mathbf{A}\right)\sum_{i=1}^n\mathbf{y}_i^2\leq\sum_{i=1}^n\lambda_i\mathbf{y}_i^2\leq \lambda_{\max}\left(\mathbf{A}\right)\sum_{i=1}^n\mathbf{y}_i^2.
    \end{equation}
    Hence,
    \begin{equation}
        \lambda_{\min}\left(\mathbf{A}\right)\|\mathbf{y}\|^2\leq\mathbf{z}^\top\mathbf{A}\mathbf{z}\leq\lambda_{\max}\left(\mathbf{A}\right)\|\mathbf{y}\|^2.
    \end{equation}
    Because $\|\mathbf{y}\|=\|\mathbf{z}\|$, we obtain
    \begin{equation}
        \lambda_{\min}\left(\mathbf{A}\right)\|\mathbf{z}\|^2\leq\mathbf{z}^\top\mathbf{A}\mathbf{z}\leq\lambda_{\max}\left(\mathbf{A}\right)\|\mathbf{z}\|^2.
    \end{equation}
    This proves the claim.
\end{proof}

\begin{proposition}[Restricted Strong Convexity]
\label{theo:strong_convex}
    Assume $\mathbf{X}$ is full rank. For any sublevel set $\mathcal{W}=\left\{\mathbf{w}:L^1\left(\mathbf{w}\right)\leq l\right\}$, there exists $\mu>0$ such that
    \begin{equation}
        L^1\left(\mathbf{v}\right)\geq L^1\left(\mathbf{w}\right)+\nabla L^1\left(\mathbf{w}\right)^\top\left(\mathbf{v}-\mathbf{w}\right)+\frac{\mu}{2}\|\mathbf{v}-\mathbf{w}\|^2,
    \end{equation}
    for all $\mathbf{w,\,v\in\mathcal{W}}$ such that $\mathbf{v}-\mathbf{w}\in\operatorname{span}\left(\mathbf{X}\right)$.
\end{proposition}

\begin{proof}
    Consider the 2nd-Taylor explansion of $L^1$ around $\mathbf{w}$,
    \begin{equation}
        \label{eq:taylor_expansion}
        L^1\left(\mathbf{v}\right)=L^1\left(\mathbf{w}\right)+\nabla L^1\left(\mathbf{w}\right)^\top\left(\mathbf{v}-\mathbf{w}\right)+\frac{1}{2}\left(\mathbf{v}-\mathbf{w}\right)\left[\nabla^2L^1\left(\bar{\mathbf{w}} \right)\right]\left(\mathbf{v}-\mathbf{w}\right),
    \end{equation}
    where $\nabla^2L^1\left(\bar{\mathbf{w}} \right)$ is the Hessian of $L^1$ evaluated at $\bar{\mathbf{w}}$, a point lying between $\mathbf{v}$ and $\mathbf{w}$. Hessian takes the form
    \begin{equation}
        \nabla^2L^1\left(\bar{\mathbf{w}} \right) = \mathbf{X}\mathbf{D}_{\bar{\mathbf{w}}}\mathbf{X}^\top,
    \end{equation}
    where
    \begin{equation}
        \mathbf{D}_{\bar{\mathbf{w}}}=\operatorname{diag}\left[\sigma\p{\mathbf{\bar w}^\top\bx_1}\p{1-\sigma\p{\mathbf{\bar w}^\top\bx_1}},\dots,\sigma\p{\mathbf{\bar w}^\top\bx_n}\p{1-\sigma\p{\mathbf{\bar w}^\top\bx_n}}\right].
    \end{equation}
    We now show that there is a constant $\omega>0$ such that
    \begin{equation}
    \label{eq:hessian_bound}
        \sigma\left(\bar{\mathbf{w}}\mathbf{x}_i\right)\left(1-\sigma\left(\bar{\mathbf{w}}\mathbf{x}_i\right)\right) \geq \omega,
    \end{equation}
    for all $\bar{\mathbf{w}}\in\mathcal{W}$ and $i\in\left\{1,\dots,n\right\}$, so that we can claim $\mathbf{D}_{\bar{\mathbf{w}}}\succeq\omega\mathbf{I}$, consequently $\nabla^2L^1\left(\bar{\mathbf{w}}\right)\succeq\omega\mathbf{X}\mathbf{X}^\top$.
    Let $\mathbf{w}\in\mathcal{W}$. The bound on $L^1\left(\mathbf{w}\right)$ implies a bound on $\ell\left(\mathbf{w}^\top\mathbf{x}_i\right)$ for all $i$,
    \begin{equation}
        \ell\left(\mathbf{w}^\top\mathbf{x}_i\right)\leq nL^1\left(\mathbf{w}\right)\leq nl.
    \end{equation}
    Because $\ell_i$ is convex and $\ell_i\left(u\right)\rightarrow\infty$ as $u\rightarrow\pm\infty$, we know that $\ell^{-1}_{i}\left((-\infty,nl]\right)$ is a bounded interval, and the finite union $\cup^{n}_{i=1}\ell^{-1}_{i}\left((-\infty,nl]\right)$ is also a bounded interval, whose size depends only on $nl$ and the data. Hence, there exists $K>0$ such that $\mathbf{w}^\top\mathbf{x}_i\in\left[-K,K\right]$ for all $\mathbf{w}\in\mathcal{W}$ and $i\in\left\{1,\dots,n\right\}$. The existence of $\omega>0$ satisfies \Equationref{eq:hessian_bound}.

    Now, let us apply $\nabla^2L^1\left(\mathbf{w}\right)\succeq\omega\mathbf{X}\mathbf{X}^\top$ a lower bound to the right-hand side in \Equationref{eq:taylor_expansion}:

    \begin{equation}
        \label{eq:right_bounded_taylor_expansion}
        L^1\left(\mathbf{v}\right)\geq L^1\left(\mathbf{w}\right)+\nabla L^1\left(\mathbf{w}\right)^\top\left(\mathbf{v}-\mathbf{w}\right)+\frac{\omega}{2}\left(\mathbf{v}-\mathbf{w}\right)^\top\mathbf{X\mathbf{X}^\top}\left(\mathbf{v}-\mathbf{w}\right).
    \end{equation}

    Consider two cases.

    $\blacksquare$ Case 1: If $n\geq d$, $\mathbf{X}\mathbf{X}^\top$ is full-rank. Applying \Lemmaref{lem:rayleigh}, we have $\mathbf{X}\mathbf{X}^\top\succeq \lambda_{\min}\mathbf{I}$ holds, where $\lambda_{\min}>0$ is the smallest eigenvalue of $\mathbf{X}\mathbf{X}^\top$. Combined with \Equationref{eq:right_bounded_taylor_expansion}, this proves the claim for $n\geq d$ and $\mu=\omega\lambda_{\min}$.

    $\blacksquare$ Case 2: If $n<d$, $\mathbf{X}^\top\mathbf{X}$ is full-rank. We use the assumption $\mathbf{v}-\mathbf{w}\in\operatorname{span}\left(\mathbf{X}\right)$ to deduce

    \begin{align}
        \|\mathbf{v}-\mathbf{w}\|^2 &=\|\mathbf{P}_\mathbf{X}\left(\mathbf{v}-\mathbf{w}\right)\|^2 \\ &= \left(\mathbf{v}-\mathbf{w}\right)^\top\mathbf{X}\left(\mathbf{X}^\top\mathbf{X}\right)^{-1}\mathbf{X}^\top\left(\mathbf{v}-\mathbf{w}\right) \\ &\leq\lambda_{\max}\left(\mathbf{v}-\mathbf{w}\right)^\top\mathbf{X}\mathbf{X}^\top\left(\mathbf{v}-\mathbf{w}\right),
    \end{align}

    where $\lambda_{\max}>0$ is the largest eigenvalue of $\left(\mathbf{X}^\top\mathbf{X}\right)^{-1}$, according to \Lemmaref{lem:rayleigh}. Combined with \Equationref{eq:right_bounded_taylor_expansion}, this proves the claim for $n<d$ and $\mu=\omega/\lambda_{\max}$.

\end{proof}

\begin{proposition}[Restricted Polyak-Lojasiewicz]
\label{prop:polyak}
    Assume $\mathbf{X}$ is full-rank. For any sublevel set $\mathcal{W}=\left\{\mathbf{w}:L^1\left(\mathbf{w}\right) \leq l \right\}$, there exists $c>0$ such that
    \begin{equation}
        cL^1\left(\mathbf{w}\right)\leq \frac{1}{2}\|\nabla L^1\left(\mathbf{w}\right)\|\,,
    \end{equation}
    for all $\mathbf{w\in \mathcal{W}}.$
\end{proposition}

\begin{proof}
Let $\mathbf{w\in\mathcal{W}}$. (If $\mathcal{W}$ is empty, the claim is trivially
true.) \Theoremref{theo:strong_convex} applied to $\mathcal{W}$ implies that for some $\mu>0$,
\begin{equation}
    L^1\left(\mathbf{v}\right)\geq L^1\left(\mathbf{w}\right)+\nabla L^1\left(\mathbf{w}\right)\left(\mathbf{v}-\mathbf{w}\right)+\frac{\mu}{2}\|\mathbf{v}-\mathbf{w}\|^2,
\end{equation}
for all $\mathbf{v}\in\mathcal{W}\cap\mathcal{V}$ where $\mathcal{V}=\left\{\mathbf{v}:\mathbf{v}-\mathbf{w}\in\operatorname{span}(\mathbf{X})\right\}$.

Taking $\min_{\mathbf{v}\in\mathcal{W}\cap\mathcal{V}}$ on both sides, then relaxing part of the constraint on the right-hand side yields

\begin{align}
    \min_{\mathbf{v\in\mathcal{W}\cap\mathcal{V}}}L^1\left(\mathbf{v}\right) &\geq \min_{\mathbf{v\in\mathcal{W}\cap\mathcal{V}}}L^1\left(\mathbf{w}\right)+\nabla L^1\left(\mathbf{w}\right)^\top\left(\mathbf{v}-\mathbf{w}\right)+\frac{\mu}{2}\|\mathbf{v}-\mathbf{w}\|^2 \\ &\geq \min_{\mathbf{v\in\mathcal{V}}}L^1\left(\mathbf{w}\right)+\nabla L^1\left(\mathbf{w}\right)^\top\left(\mathbf{v}-\mathbf{w}\right)+\frac{\mu}{2}\|\mathbf{v}-\mathbf{w}\|^2.
\end{align}

\end{proof}

\subsection{Property of Low-Capacity Distilled Learning}
\begin{theorem}[Projected-Data Equivalence under a Capacity Constraint]
\label{thm:projected_data_equiv}
    When optimization is restricted to the student space \(\mathcal{F}\), the empirical risk depends on the training inputs only
through their projections onto \(\mathcal{F}\).
\end{theorem}

\begin{proof}
We derive the loss for a student model parameter as follows
    \begin{align}
        L^1(\mathbf{w}_\mathcal{F})=L^1(P_{\mathcal{F}}\mathbf{w}) &= \frac{1}{n}\sum_{i=1}^n \ell_i((P_{\mathcal{F}}\mathbf{w})^\top \mathbf{x}_i) \\ &= \frac{1}{n}\sum_{i=1}^n \ell_i(\mathbf{w}^\top \left(P_{\mathcal{F}}^\top\mathbf{x}_i)\right) \\ &= \frac{1}{n}\sum_{i=1}^n \ell_i(\mathbf{w}^\top \left(P_{\mathcal{F}}\mathbf{x}_i)\right).
    \end{align}

The last equation is the loss of the teacher model on the data projected on the student model space.
\end{proof}

\textbf{Interpretation.} The theorem formalizes the statement that a low-capacity linear student can only exploit the component of the data lying in its parameter space \(\mathcal{F}\). Any component of \(\mathbf{x}_i\) in \(\mathcal{F}^\perp\) is invisible to the student, because for every \(\mathbf{u} \in \mathcal{F}\),

\begin{equation}
    \mathbf{u}^\top \mathbf{x}_i = (P_\mathcal{F}\mathbf{u}^\top) \mathbf{x}_i = \mathbf{u}^\top (P_{\mathcal{F}}\mathbf{x}_i).    
\end{equation}
Thus, the effective data span available to the student is

\begin{equation}
    \operatorname{span}(P_{\mathcal{F}}\mathcal{X}) = P_{\mathcal{F}}\,\operatorname{span}(\mathcal{X}).    
\end{equation}
In particular, when distilling from a larger model into a student whose
parameter space is \(\mathcal{F}\), the student can only learn teacher
behavior through the component of the training data that survives this
projection.

\begin{theorem}[Space of Distilled Models]
\label{theo:student_space}
Assume the student is a directly parameterised
linear classifier (N=1). Then,
\begin{equation}
    \mathbf{w}(\tau)\in \mathbf{w}(0) + \operatorname{span}(\mathbf X)\,,
\end{equation}
for $\tau\in[0,-\infty)$, and $\mathbf{w}(0)$ is the initialized model parameter.
\end{theorem}

\begin{proof}

From \Theoremref{theo:grad_constraint}, the gradient satisfies
\begin{equation}
    \nabla L^1(\mathbf{w})\in \operatorname{span}(\mathbf{X}),
\end{equation}
for every $\mathbf{w}$. Under gradient flow for ($N=1$),
\begin{equation}
    \frac{\textrm{d}}{\textrm{d}\tau}\mathbf{w}(\tau)=-\nabla L^1(\mathbf{w}(\tau)),
\end{equation}
so the velocity vector ($\dot{\mathbf{w}}(\tau)$) always lies in ($\operatorname{span}(\mathbf{X})$). 

Define
\begin{equation}
    \mathbf{u}(\tau):=\mathbf{w}(\tau)-\mathbf{w}(0).
\end{equation}
Then
\begin{equation}
    \mathbf{u}(0)=0,\qquad \frac{\dd}{\dd\tau}\mathbf{u}(\tau)=\frac{\dd}{\dd\tau}\mathbf{w}(\tau)=-\nabla L^1(\mathbf{w}(\tau))\in \operatorname{span}(\mathbf{X}).
\end{equation}
Since $\mathbf{u}(0)=0\in \operatorname{span}(\mathbf{X})$ and its derivative always lies in $\operatorname{span}(\mathbf{X})$, it follows that $\mathbf{u}(\tau)\in \operatorname{span}(\mathbf{X})$ for all $\tau\ge 0$. Therefore,
\begin{equation}
    \mathbf{w}(\tau)=\mathbf{w}(0)+\mathbf{u}(\tau)\in \mathbf{w}(0)+\operatorname{span}(\mathbf{X}).
\end{equation}
So the trajectory is not confined to the linear subspace $\operatorname{span}(\mathbf{X})$ anymore; it is confined to the affine subspace obtained by translating $\operatorname{span}(\mathbf{X})$ by the initial point $\mathbf{w}(0)$.
\end{proof}

\subsection{Optimal Distillation Solution of a One-Layer Linear Student Model}
\label{appendix:one_layer_network}

\begin{theorem}[High-capacity One-layer Student Solution]
\label{appendix:high_capacity_one_layer_sol}
    Assume the student is a directly parameterised linear classifier (N = 1). Then, the student’s weight vector almost surely coverges to $\hat{\mathbf{w}}$, where
    \begin{equation}
        \hat{\mathbf{w}} = \mathbf{w}_0+P_{\mathbf{X}}\left(\mathbf{w}_\ast-\mathbf{w}_0\right),
    \end{equation}
    for $\tau\rightarrow\infty$.
\end{theorem}

\begin{proof}
Because $\mathbf{w}\left(\tau\right)$ evolve according to gradient flow,
\begin{equation}
    \frac{\textrm{d}}{\textrm{d}\tau} \mathbf{w}(\tau) = - \nabla L^1(\mathbf{w}(\tau));
\end{equation}
hence,
\begin{equation}
\label{eq:mono_decrease}
L'\left(\tau\right)=\frac{\dd}{\dd\tau}L^1\left(\mathbf{w}\left(\tau\right)\right)=\nabla L^1\left(\mathbf{w}\left(\tau\right)\right)^\top \frac{\dd}{\dd\tau}\mathbf{w}\left(\tau\right)=-\|\nabla L^1\left(\mathbf{w}\left(\tau\right)\right)\|^2.
\end{equation}
The data matrix $\mathbf{X}$ is full-rank, we can therefore apply \Propositionref{prop:polyak} to $\mathcal{W}=\left\{\mathbf{w}:L^1\left(\mathbf{w}\right)\leq L^1\left(\textbf{0}\right)\right\}$ and $\mathbf{w}\left(\tau\right)$ to lower-bound the gradient norm on the right-hand side of \Equationref{eq:mono_decrease}. We obtain

\begin{equation}
    L'\left(\tau\right)\leq -cL\left(\tau\right)
\end{equation}
for some $c>0$ and all $\tau\in[0,\infty)$, or equivalently,

\begin{equation}
    \left(\log L\left(\tau\right)\right)'\leq c.
\end{equation}
Integrating over $\left[0,t\right]$ yields $L\left(t\right)\leq L\left(0\right)\,\cdot\,e^{-ct}$, which proves global convergence in the objective $L\left(t\right)\to0$ as $t\to\infty$.

Now invoke \Propositionref{theo:strong_convex} with $\mathcal{W}$ as above, $\mathbf{v}=\mathbf{w}\left(t\right)$ and $\mathbf{w}=\hat{\mathbf{w}}$ (we know that both $\mathbf{w}\left(\tau\right),\hat{\mathbf{w}}\in\mathcal{W}\cap\operatorname{span}\left(\mathbf{X}\right)$, partly by \Theoremref{theo:student_space}):

\begin{equation}
    L\left(t\right)\geq \frac{\mu}{2}\|\mathbf{w}\left(t\right)-\mathbf{w}\|^2.
\end{equation}
Since $L\left(t\right)\to 0$ as $t\to\infty$, then the theorem follows.
\end{proof}

\textbf{Interpretation.} \Theoremref{appendix:high_capacity_one_layer_sol} shows that training can only transfer teacher knowledge that is exposed by the data. The vector \(\mathbf{w}_\ast-\mathbf{w}_0\) represents the knowledge gap between teacher and student. The projection term \(P_{\mathbf X}(\mathbf{w}_\ast-\mathbf{w}_0)\) indicates that the knowledge gap between the teacher and the student is transferred only through its component in the data span. Hence, learning is fundamentally constrained to the subspace revealed by the training data. When the data span covers the entire knowledge gap, the student can fully learn the teacher knowledge; otherwise, it can recover only the component of that gap that is identifiable from the data. Thus, the data act as a mirror of the teacher: if the data span aligns well with the knowledge gap, the student learns effectively; if it is orthogonal to that gap, then no transfer occurs. Moreover, directions orthogonal to the data span are never updated, so the final student remains partially constrained by its initialization \(\mathbf{w}_0\).

\begin{theorem}[Low-capacity One-layer Student Solution]
\label{theo:low-cap_one-layer}
    Assume the low-capacity student is a directly parameterised linear classifier (N = 1). Then, the student’s weight vector converges to $\hat{\mathbf{w}}_{\mathcal{F}}$, where
    \begin{equation}
        \hat{\mathbf{w}}_{\mathcal{F}} =  \mathbf{w}_0+P_{\mathbf{X}_\mathcal{F}}\left(\mathbf{w}_\ast-\mathbf{w}_0\right),
    \end{equation}
    for $\tau\rightarrow\infty$.
\end{theorem}

\begin{proof}
By \Theoremref{thm:projected_data_equiv}, under the capacity constraint \(\mathcal F\), training the one-layer student on the original dataset is equivalent to training an unconstrained one-layer student on the projected data \(P_{\mathcal F}X\). In particular, the effective data span becomes
\[
\mathbf X_{\mathcal F} := \operatorname{span}(P_{\mathcal F}\mathbf X)
= P_{\mathcal F}\operatorname{span}(\mathbf X).
\]

Now apply \Theoremref{appendix:high_capacity_one_layer_sol} to this projected problem. Since \Theoremref{appendix:high_capacity_one_layer_sol} states that for a directly parameterized one-layer linear classifier, gradient flow converges to the initial point plus the orthogonal projection of the teacher--student gap onto the data span, we obtain
\[
\mathbf w(\tau)\to \mathbf  w_0 + P_{\mathbf  X_{\mathcal F}}(\mathbf  w^\ast - \mathbf  w_0)
\qquad \text{as } \tau\to\infty.
\]

This is exactly the claim of \Theoremref{theo:low-cap_one-layer}.
\end{proof}

\subsection{Optimal Distillation Solution of a Deep Linear Student Model}
\label{appendix:deep_linear_network}
\begin{theorem}[High-capacity Deep Linear Student Solution]
\label{theo:high-capcity_deep}
    Let $\hat{\mathbf{w}}$ be defined as in \Theoremref{appendix:high_capacity_one_layer_sol}. Assume the student is a deep linear network $\mathbf{w}$, initialized such that for some $\epsilon>0$, satisfying the conditions below,
    \begin{equation}
    \label{eq:initialization_cond}
        \|\mathbf{w}\left(0\right)\|<\min\left\{\|\hat{\mathbf{w}}\|,\epsilon^N\left(\epsilon^2\|\hat{\mathbf{w}}\|^{-\frac{2}{N}}+\|\hat{\mathbf{w}}\|^{2-\frac{2}{N}}\right)^{-\frac{N}{2}}\right\},
    \end{equation}
    \begin{equation}
        \label{eq:sub-level_assumption}L^1\left(\mathbf{w}\left(0\right)\right)<L^1\left(\mathbf{0}\right),
    \end{equation}
    \begin{equation}
        \label{eq:balance_condition}\mathbf{W}_{j+1}\left(0\right)^\top\mathbf{W}_{j+1}\left(0\right)=\mathbf{W}_{j}\left(0\right)\mathbf{W}_{j}\left(0\right)^\top
    \end{equation}
    for $j=1,\dots,N-1$. Then, for $n
    \geq d$, student's weight vector converges
    \begin{equation}
        \|\mathbf{w}\left(t\right)-\hat{\mathbf{w}}\|\leq\epsilon,
    \end{equation}
    for all $t$ large enough.
\end{theorem}

\begin{proof}
For the proof, we need the result by \citep{arora2018optimization}, characterizing the induced flow on $\mathbf{w}\left(\tau\right)$ when running gradient descent on the component matrices $\mathbf{W}_i$.

\begin{lemma}[Claim 2 in \citep{arora2018optimization}]
    If the ballancedness condition \Equationref{eq:balance_condition} holds, then
    \begin{equation}
        \frac{\partial\mathbf{w}\left(\tau\right)}{\partial\tau}=-\|\mathbf{w}\left(\tau\right)\|^\frac{2\left(N-1\right)}{N}\left(\left(\nabla L^1\left(\tau\right)\right)+\left(N-1\right)\cdot\mathbf{P}_{\mathbf{w}\left(\tau\right)}\nabla L^1\left(\mathbf{w}\left(\tau\right)\right)\right).
    \end{equation}
\end{lemma}

We start by looking at the time-derivative of $L$,

\begin{align}
    L'\left(\tau\right) &= \nabla L^1\left(\mathbf{w}\left(\tau\right)\right)^\top\left(\frac{\partial\mathbf{w}\left(\tau\right)}{\partial\tau}\right) \\ &= - \|\mathbf{w}\left(\tau\right)\|^{\frac{2\left(N-1\right)}{N}}\left(\|\nabla L^1\left(\mathbf{w}\left(\tau\right)\right)\|^2 + \left(N-1\right)\cdot\|\mathbf{P}_{\mathbf{w}\left(\tau\right)}\nabla L^1\left(\mathbf{w}\left(\tau\right)\right)\|^2\right) \\ & \label{eq:dloss_bound} \leq - \|\mathbf{w}\left(\tau\right)\|^{\frac{2\left(N-1\right)}{N}} \cdot \|\nabla L^1\left(\mathbf{w}\left(\tau\right)\right)\|^2 
\end{align}

It is non-positive, so $\mathbf{w}\left(\tau\right)$ stays within the $L\left(\mathbf{w}\left(0\right)\right)$-sublevel set thoughout optimization,
\begin{equation}
    \mathbf{w}\left(\tau\right)\in\mathcal{W}=\left\{\mathbf{w}:L^1\left(\mathbf{w}\right) \leq L\left(\mathbf{w}\left(0\right)\right)\right\}.
\end{equation}

Also, $\mathcal{W}$ is convex and by assumption shown in \Assumptionref{eq:sub-level_assumption} it does not contain $\mathbf{0}$. We can therefore take $\delta>0$ to be the distance between $\mathcal{W}$ and $\mathbf{0}$, and it follows that $\|\mathbf{w}\left(\tau\right)\|\geq\delta$ for $\tau\in[0,\infty)$.

Now, noting that $\mathbf{X}$ is full-rank, apply \Propositionref{prop:polyak} to $\mathcal{W}$ and $\mathbf{w}\left(\tau\right)$ to upper-bound the right-hand side of \Inequalityref{eq:dloss_bound},

\begin{equation}
    L'\left(\tau\right)\leq-c\delta^{\frac{2\left(N-1\right)}{N}}L\left(\tau\right)
\end{equation}

Letting $\tilde{c}=c\delta^\frac{2\left(N-1\right)}{N}$, we get $\left(\log L\left(\tau\right)\right)'\leq-\tilde{c}$ and consequently $L\left(t\right)\leq L\left(0\right)\cdot e^{-\tilde{c}t}$. This proves convergence in the objective, $L\left(t\right)\to 0$ as $t\to\infty$.

To prove convergence in parameters, we decompose the current-optimal gap $\mathbf{w}\left(\tau\right)-\hat{\mathbf{w}}$ into orthogonal components and bound each of them separately,

\begin{equation}
\label{eq:param_decompose}
    \|\mathbf{w}\left(\tau\right)-\hat{\mathbf{w}}\|^2 = \|\mathbf{P}_{\mathbf{X}}\left(\mathbf{w}\left(\tau\right)-\hat{\mathbf{w}}\right)\|^2 + \|\mathbf{P}_{\mathbf{Q}}\left(\mathbf{w}\left(\tau\right)-\hat{\mathbf{w}}\right)\|^2,
\end{equation}
where the columns of $\mathbf{Q}\in\mathbb{R}^{d\times\left(d-n\right)}$ orthogonally complement those of $\mathbf{X}$. If $n\geq d$, we simply bound the first term and disregard the second one.

To bound the first term, invoke the \Propositionref{theo:strong_convex} with $\mathcal{W}$, $\mathbf{v}=\mathbf{P}_{\mathbf{X}}\mathbf{w}\left(\tau\right)$ and $\mathbf{w}=\mathbf{P}_{\mathbf{X}}\hat{\mathbf{w}}$. One can check that $L^1\left(\mathbf{P}_\mathbf{X}\mathbf{u}\right)=L^1\left(\mathbf{u}\right)$ for all $\mathbf{u}\in\mathbb{R}^d$, so $\mathbf{P}_{\mathbf{X}}\mathbf{w}\left(\tau\right)\in\mathcal{W}$ and our use of the theorem is legal. We obtain
\begin{equation}
    L\left(\tau\right)\geq \frac{\mu}{2}\|\mathbf{P}_{\mathbf{X}}\left(\mathbf{w}\left(\tau\right)-\hat{\mathbf{w}}\right)\|^2.
\end{equation}

Since $L(\tau)\to 0$, it follows that
\begin{equation}
    \label{eq:convergence_X}\|\mathbf{P}_{\mathbf{X}}\left(\mathbf{w}\left(\tau\right)-\hat{\mathbf{w}}\right)\|^2\to 0
\end{equation}
as $\tau\to\infty$.

For the second term, notice that $\hat{\mathbf{w}}\in\operatorname{span}\left(\mathbf{X}\right)$, so $\mathbf{P}_{\mathbf{Q}}\left(\hat{\mathbf{w}}\right)$ vanishes and we are left with $\|\mathbf{P}_{\mathbf{Q}}\mathbf{w}\left(\tau\right)\|^2$. Denote this quanity $q\left(\tau\right)$. Its time derivative is
\begin{align}
    q'\left(\tau\right) &= 2\left(\mathbf{P}_\mathbf{Q}\mathbf{w}\left(\tau\right)\right)^\top\left(\frac{\partial\mathbf{w}\left(\tau\right)}{\partial\tau}\right) \\ &= -2\|\mathbf{w}(\tau)\|^{\frac{2(N-1)}{N}}\Biggl(\mathbf{w}(\tau)^\tau\mathbf{P}_\mathbf{Q}\nabla L^1(\mathbf{w}(\tau)) \\ & \qquad\qquad + \frac{N-1}{\|\mathbf{w}(\tau)\|^2}\cdot\mathbf{w}(\tau)^\top\mathbf{P}_\mathbf{Q}\mathbf{w}(\tau)\cdot\mathbf{w}(\tau)^\top\nabla L^1(\mathbf{w}(\tau))\Biggr)
    \\ &= -2q(\tau)(N-1)\|\mathbf{w}(\tau)\|^\frac{-2}{N}\mathbf{w}(\tau)^\top\nabla L^1(\mathbf{w}(\tau)),
\end{align}

where we have used the fact that $\nabla L^1(\mathbf{w}(\tau))\in\operatorname{span}(\mathbf{X})$ stated in \Theoremref{theo:grad_constraint} and $\mathbf{Q}$ is orthogonal to $\mathbf{X}$. Rearranging, we obtain
\begin{equation}
    \frac{\dd}{\dd\tau}\left(\frac{\log q(\tau)}{2(N-1)}\right) = -\|\mathbf{w}(\tau)\|^{-\frac{2}{N}}\cdot\mathbf{w}(\tau)^\top\nabla L^1(\mathbf{w}(\tau)).
\end{equation}
It turns out that the right-hand side expression is integrable in yet another way, namely
\begin{equation}
    \frac{\dd}{\dd\tau}\left(\frac{1}{2N}\log\|\mathbf{w}(\tau)\|^2\right) = - \|\mathbf{w}(\tau)\|^{-\frac{2}{N}}\cdot\mathbf{w}(\tau)^\top\nabla L^1(\mathbf{w}(\tau)).
\end{equation}

Equating the two and integrating over $[0,t]$ yields

\begin{equation}
    \log\frac{q(t)}{q(0)}=\frac{N-1}{N}\cdot\log\frac{\|\mathbf{w}(t)\|^2}{\|\mathbf{w}(0)\|^2},
\end{equation}
which implies
\begin{equation}
\label{eq:bound_q}
    \frac{q(t)}{\|\mathbf{w}(t)\|^2}\leq\left(\frac{\|\mathbf{w}(0)\|}{\|\mathbf{w}(t)\|}\right)^\frac{2}{N},
\end{equation}
because $q(0)\leq \|\mathbf{w}(0)\|^2$.

We now bound the norm of $\mathbf{w}(t)$. Starting from an orthogonal decompositon similar to \Equationref{eq:param_decompose} and applying convergence on $\mathbf{X}$ space in \Equationref{eq:convergence_X} and bound for parameter on $\mathbf{X}$-orthogonal space in \Equationref{eq:bound_q}, we get
\begin{equation}
    \|\mathbf{w}(t)\|^2=\|\mathbf{P}_\mathbf{X}\mathbf{w}(t)\|^2 + \|\mathbf{P}_\mathbf{Q}\mathbf{w}(t)\|^2
\end{equation}
\begin{equation}
    \limsup_{t \to \infty}\|\mathbf{w}(t)\|^2\leq\|\hat{\mathbf{w}}\|^2+\|\mathbf{w}(0)\|^\frac{2}{N}\limsup_{t \to \infty}\|\mathbf{w}(t)\|^{2-\frac{2}{N}} 
\end{equation}

Denote $\nu:=\limsup_{t \to \infty}\|\mathbf{w}(t)\|$. By the same orthogonal decomposition, we also know that $\nu^2\geq\limsup_{t \to \infty}\|\mathbf{P_\mathbf{X}}\mathbf{w}(t)\|^2=\|\hat{\mathbf{w}}\|^2$, so we can divide both sides above by $\nu^2$,
\begin{equation}
    1\leq\frac{\|\hat{\mathbf{w}}\|^2}{\nu^2}+\frac{\|\mathbf{w}(0)\|^\frac{2}{N}}{\nu^\frac{2}{N}}=:f(\nu).
\end{equation}
On the right-hand side, we now have a decreasing function of $\nu$ that goes to zero as $\nu\to\infty$. However,evaluated at our specific $\nu$, it is lower-bounded by 1, implying an implicit upper bound for $\nu$.

How do we find this bound? Suppose we find some constant $K$ such that $f(K)\leq1$. Then , because $f$ is decreasing, it must be the case that $\nu\leq K$. One such candidate for $K$ is
\begin{equation}
    K=\|\hat{\mathbf{w}}\|\cdot\left(1-\frac{\|\mathbf{w}(0)^\frac{2}{N}\|}{\|\hat{\mathbf{w}}\|^\frac{2}{N}}\right)^{-\frac{N}{2(N-1)}}.
\end{equation}

To check that indeed $f(K)\leq 1$, start from the inequality

\begin{equation}
    \left(\frac{\|\hat{\mathbf{w}}\|}{K}\right)^{\frac{2(N-1)}{N}}+\left(\frac{\|\mathbf{w}(0)\|}{\|\hat{\mathbf{w}}\|}\right)^{\frac{2}{N}}=1\leq\left(1-\left(\frac{\|\mathbf{w}(0)\|}{\|\hat{\mathbf{w}}\|}\right)^{\frac{2}{N}}\right)^{-\frac{1}{N-1}}=\left(\frac{\|\hat{\mathbf{w}}\|}{K}\right)^{-\frac{2}{N}}.
\end{equation}

Taking the leftmost and rightmost expression and multiply by $(\|\hat{\mathbf{w}}\|/K)^{2/N}$ yields
\begin{equation}
    f(K)=\frac{\|\hat{\mathbf{w}}\|^2}{K^2}+\frac{\|\mathbf{w}(0)\|^{\frac{2}{N}}}{K^{\frac{2}{N}}}\leq 1.
\end{equation}

Hence,

\begin{equation}
\label{eq:bound_wt}
    \limsup_{t\to\infty}\|\mathbf{w}(t)\|\leq\|\hat{\mathbf{w}}\|\cdot\left(1-\left(\frac{\|\mathbf{w}(0)\|}{\|\hat{\mathbf{w}}\|}\right)^\frac{2}{N}\right)^{-\frac{N}{2(N-1)}}.
\end{equation}

Finally, let us turn back to our original goal of bounding $\|\mathbf{w}(\tau)-\hat{\mathbf{w}}\|^2$. With \Equationref{eq:param_decompose}, \Equationref{eq:convergence_X}, \Equationref{eq:bound_q}, \Equationref{eq:bound_wt}, we now know that
\begin{align}
    \limsup_{t\to\infty}\|\mathbf{w}(\tau)-\hat{\mathbf{w}}\|^2 &\leq \|\mathbf{w}(0)\|^\frac{2}{N}\|\hat{\mathbf{w}}\|^\frac{2(N-1)}{N}\left(1-\left(\frac{\mathbf{w}(0)}{\hat{\mathbf{w}}}^\frac{2}{N}\right)\right)^{-1} \\ &= \frac{\|\hat{\mathbf{w}}\|^{2+\frac{2}{N}}}{\|\hat{\mathbf{w}}\|^\frac{2}{N}-\|\mathbf{w}(0)\|^\frac{2}{N}}-\|\hat{\mathbf{w}}\|^2.
\end{align}

Hence, if we initialize close enough to zero, as specified by condition in \Inequalityref{eq:initialization_cond}, we ensure that
\begin{equation}
    \limsup_{t\to\infty}\|\mathbf{w}(\tau)-\hat{\mathbf{w}}\|^2<\epsilon^2.
\end{equation}
\end{proof}

\begin{theorem}[Low-capacity Deep Linear Student Solution]
\label{theo:low-capcity_deep}
Let
\[
\hat{\mathbf{w}}_{\mathcal F}
=
\mathbf{w}_0 + P_{X_{\mathcal F}}(\mathbf{w}^\ast - \mathbf{w}_0)
\]
be defined as in \Theoremref{theo:low-cap_one-layer}. Under the same assumptions as in \Theoremref{theo:high-capcity_deep}, suppose the student is a low-capacity deep linear network with predictor \(\mathbf{w}(\tau)\). Then, for every \(\epsilon > 0\), there exists \(\tau_\epsilon > 0\) such that
\[
\|\mathbf{w}(\tau)-\hat{\mathbf{w}}_{\mathcal F}\| \le \epsilon,
\qquad \forall \tau \ge \tau_\epsilon .
\]
\end{theorem}

\begin{proof}
By \Theoremref{appendix:high_capacity_one_layer_sol}, under the capacity constraint \(\mathcal F\), training the low-capacity deep linear student on the original dataset is equivalent to training an unconstrained deep linear student on the projected data \(P_{\mathcal F}\mathbf X\). In particular, the effective data span becomes
\[
\mathbf X_{\mathcal F}:=\operatorname{span}(P_{\mathcal F}\mathbf  X)=P_{\mathcal F}\operatorname{span}(\mathbf X).
\]
Now apply \Theoremref{theo:high-capcity_deep} to this projected problem. \Theoremref{theo:high-capcity_deep} states that, under the same initialization assumptions, the predictor of a deep linear network converges arbitrarily close to the optimal solution associated with the effective data span. Therefore, for every \(\epsilon>0\), there exists \(\tau_\epsilon>0\) such that
\[
\|\mathbf{w}(\tau)-\hat{\mathbf{w}}_{\mathcal F}\|\le \epsilon,
\qquad \forall \tau\ge \tau_\epsilon,
\]
where
\[
\hat{\mathbf{w}}_{\mathcal F}
=
\mathbf{w}_0+P_{X_{\mathcal F}}(\mathbf{w}^\ast-\mathbf{w}_0).
\]
This is exactly the claim of \Theoremref{theo:low-capcity_deep}.
\end{proof}

\section{Experiment Setup}
\label{appendix:setup}
\subsection{Training}
\label{appendix:training}
We fine-tune each student model using full-parameter supervised fine-tuning (SFT) on the selected training subset, such as the high-likelihood or low-likelihood data. Concretely, training starts from the corresponding pretrained checkpoint (e.g., \texttt{Qwen/Qwen2.5-1.5B}) and uses full finetuning rather than parameter-efficient adaptation. All experiments are run on 2 NVIDIA A100 GPUs. To support long-context training efficiently, we use DeepSpeed ZeRO-3, FlashAttention-2, gradient checkpointing, and bfloat16 precision. The maximum sequence length is set to 8192 tokens, and data are formatted with the Qwen template. For optimization, we use a per-device batch size of 1, gradient accumulation steps of 1, a learning rate of $1.25\times10^{-6}$, cosine learning-rate decay, no warmup, and train for 5 epochs. We log training statistics at every step and save checkpoints every 1110 steps while storing model weights only. This training configuration is fixed across data-selection settings so that any performance differences can be attributed to the choice of training data rather than changes in optimization or hardware setup.

\section{Extended Results}
\label{appendix:additional_results}

\subsection{Small Qwen models show {\color{SMALLCOLOR}\textbf{Fast-Fit}} without eventual crossover}

\Tableref{tab:qwen25_15b_high_low} and \Tableref{tab:qwen25_3b_high_low} show {\color{SMALLCOLOR}\textbf{Fast-Fit}} without eventual crossover for small Qwen models. The 1.5B result is the cleanest example of the capacity bottleneck. For \texttt{Qwen2.5-1.5B}, high-likelihood training jumps from an Avg of 0.58 at initialization to 21.26 after only 0.25 epoch, then reaches 24.16 by epoch 1. Low-likelihood training reaches only 12.13 at 0.25 and 14.77 at epoch 1. More importantly, another four epochs do not close the gap: its best is only 16.50 at epoch 2.25. \texttt{Qwen2.5-3B} is slightly different. Low likelihood exhibits a real delayed gain: its early best Avg is 19.93, whereas its eventual best reaches 24.63, a +4.70 point improvement. But high likelihood remains far ahead at 33.38.

\subsection{\texttt{Qwen2.5-7B} is the actual transition regime}

\Tableref{tab:qwen25_7b_high_low} shows \texttt{Qwen2.5-7B} appears to occupy a \textbf{transitional capacity regime} between small models that cannot effectively exploit low-likelihood supervision and larger models that benefit from it consistently. During the early stage of training, high-likelihood data has a clear advantage, achieving a best average score of 44.09 within the first epoch compared with 37.05 for low-likelihood data. However, the low-likelihood trajectory continues to improve substantially with additional optimization, rising to 43.07 at epoch 1.25, 45.02 at epoch 2.25, and eventually reaching 46.19 at epoch 4.75, exceeding the best high-likelihood result of 45.41 and achieving higher peak performance on six of the ten benchmarks. Importantly, this crossover is not persistent across all checkpoints: low-likelihood training outperforms high-likelihood training in average score at epochs 2.25, 2.5, 4.0, 4.5, and 4.75 but falls below it again at several intermediate checkpoints. This non-monotonic behavior suggests that the 7B model has sufficient representational capacity to extract useful knowledge from more difficult, low-likelihood supervision, yet optimization remains challenging enough that these gains are unstable and checkpoint-sensitive. Rather than simply demonstrating that larger models benefit from low-likelihood data, the 7B results therefore reveal an intermediate regime characterized by \textbf{intermittent access to difficult knowledge}, bridging the failure of low-capacity models and the more sustained {\color{LARGECOLOR}\textbf{Slow-Gain}} behaviour observed in higher-capacity models.

\subsection{\texttt{Qwen3-8B} gives the strongest evidence for {\color{LARGECOLOR}\textbf{Slow-Gain}}}

\Tableref{tab:qwen3_8b_high_low} illustrates that \texttt{Qwen3-8B} provides the clearest evidence for the {\color{LARGECOLOR}\textbf{Slow-Gain}} regime. High-likelihood training begins with a substantial early advantage, reaching an average score of 66.16 at epoch 0.25 compared with 57.72 for low-likelihood training. However, the low-likelihood trajectory steadily closes this gap, overtakes the high-likelihood trajectory at epoch 2.25, and remains superior at every subsequent checkpoint through epoch 5 in terms of simple average score. The contrast between early and eventual performance is particularly revealing: the best high-likelihood average improves only marginally from 66.16 in the early stage to 66.64 overall, a gain of just 0.48 points, whereas the low-likelihood average rises from an early best of 62.62 to 70.69, corresponding to an 8.07-point improvement. This pattern closely matches the proposed {\color{SMALLCOLOR}\textbf{Fast-Fit}} / {\color{LARGECOLOR}\textbf{Slow-Gain}} dynamics: high-likelihood supervision enables rapid early acquisition but leaves relatively little room for further improvement, while low-likelihood supervision starts from a weaker region yet provides substantially greater long-term gains. One subtlety is that the conclusion depends somewhat on the aggregation metric: at epoch 0.5, the simple average favors high-likelihood training (64.79 vs. 62.62), whereas the weighted average favors low-likelihood training (74.65 vs. 70.17). For this reason, the {\color{SMALLCOLOR}\textbf{Fast-Fit}} interpretation is most convincingly supported by the simple average together with the per-benchmark comparison, which shows that high-likelihood training dominates most individual tasks early, while low-likelihood training becomes consistently stronger after sufficient optimization.

\subsection{Generalization to Llama family}

The Llama results in \Tableref{tab:llama32_3b_high_low} and \Tableref{tab:llama31_8b_high_low} provide important cross-family evidence that the observed capacity-dependent effect is not specific to Qwen models. \texttt{Llama3.2-3B} closely follows the behavior of the smaller Qwen students: high-likelihood training achieves the better maximum performance on 9 of the 10 benchmarks, while the only nominal advantage for low-likelihood training occurs on Minerva and is extremely small, 20.82 versus 20.77. Moreover, low-likelihood training never surpasses high-likelihood training at the same checkpoint in either the simple or weighted average, providing strong evidence that low-capacity models consistently benefit more from supervision that is closer to their current distribution. In contrast, \texttt{Llama3.1-8B} exhibits the opposite pattern: low-likelihood training achieves higher peak performance on 6 of the 10 benchmarks, reaches a best average score of 20.93 compared with 20.10 for high-likelihood training, and outperforms high-likelihood training in simple average at 16 of the 20 checkpoints. Together, these results reproduce the same capacity-dependent trend across a different model family: high-likelihood supervision is more effective for 3B-scale students, whereas low-likelihood supervision becomes increasingly competitive and can ultimately outperform it at 8B scale. However, \texttt{Llama3.1-8B} also introduces an important qualification to the {\color{SMALLCOLOR}\textbf{Fast-Fit}} / {\color{LARGECOLOR}\textbf{Slow-Gain}} interpretation, since high-likelihood performance does not necessarily peak during the earliest stage and continues to improve from 15.45 at epoch 1 to 20.10 at epoch 3.75. Therefore, a more precise characterization is that \textbf{high-likelihood supervision establishes an early learning advantage, whereas low-likelihood supervision exhibits larger delayed gains as model capacity increases and may eventually overtake high-likelihood training}.

\subsection{Best-Checkpoint vs. Trajectory-Level Comparison of 7B and 8B via Peak Accessibility and Persistent Superiority Metrics}

An important nuance is that \Tableref{tab:results_five_epochs} reports a \emph{best-checkpoint envelope}: each benchmark score is selected independently as the best result among 20 checkpoints, and therefore the ten reported scores within a column do not necessarily correspond to the same model checkpoint. This distinction is particularly relevant for \texttt{Qwen2.5-7B}. Although low-likelihood training achieves higher benchmark-wise peak performance on 6 out of 10 datasets, suggesting a clear reversal relative to high-likelihood training, a checkpoint-aligned comparison reveals a more nuanced pattern. In terms of the simple average, low-likelihood training outperforms high-likelihood training at only 5 out of 20 checkpoints, while their average performance after the first epoch is approximately comparable and the low-likelihood trajectory exhibits greater temporal variability. These results suggest two distinct notions of improvement, which can be formalized as the following conditions:
\[
\textbf{Peak Accessibility:} \qquad \max_{t} P_{\mathrm{low}}(t) > \max_{t} P_{\mathrm{high}}(t)
\]
\[
\textbf{Persistent Superiority:} \qquad P_{\mathrm{low}}(t) > P_{\mathrm{high}}(t), \quad \forall t \in \mathcal{T}_{\mathrm{late}}
\]
where $\mathcal{T}$ denotes the set of all evaluated training checkpoints, while $t_{\mathrm{early}}$ denotes the boundary of the early-training regime. The first epoch ($t \leq 1$) is treated as the early stage, so $\mathcal{T}_{\mathrm{late}}$ contains all checkpoints after the first epoch. Thus, \emph{persistent superiority} is a stronger condition than merely achieving a higher peak: it requires low-likelihood training to remain better than high-likelihood training throughout the entire post-early-training region, rather than outperforming it at only one or a few isolated checkpoints. \texttt{Qwen2.5-7B} strongly satisfies the former criterion but only weakly satisfies the latter, suggesting that it lies in a transitional capacity regime where difficult supervision can unlock stronger solutions but remains optimization-sensitive. In contrast, Qwen3-8B satisfies both criteria: after the crossover, low-likelihood training consistently exceeds high-likelihood training across subsequent checkpoints, providing stronger evidence of a sustained {\color{LARGECOLOR}\textbf{Slow-Gain}} regime rather than merely a higher best-checkpoint outcome.

\subsection{From Fast-Fit to Sustained Slow-Gain: A Capacity-Dependent Transition}

Overall, the extended results in {\textcolor{Tablecolor}{Tables~4--9}} suggest that likelihood-based data selection induces not merely two, but three distinct learning regimes as model capacity increases. In the \emph{low-capacity regime}, represented by smaller models such as \texttt{Qwen2.5-1.5B}, \texttt{Qwen2.5-3B}, and \texttt{Llama3.2-3B}, high-likelihood supervision exhibits a clear \emph{Fast-Fit / No Catch-up} pattern: the model benefits rapidly from examples close to its current distribution, whereas low-likelihood supervision remains difficult to absorb and fails to close the performance gap even with additional training. In the \emph{intermediate-capacity regime}, exemplified most clearly by \texttt{Qwen2.5-7B}, the dynamics shift to \emph{Fast-Fit / Unstable Slow-Gain}. High-likelihood data still provides the stronger early optimization signal, but low-likelihood training eventually reaches higher-performing solutions at selected checkpoints; however, this advantage is intermittent rather than persistent, indicating that the model has sufficient capacity to exploit difficult supervision but remains sensitive to optimization dynamics. Finally, in the \emph{high-capacity regime}, as observed most prominently for \texttt{Qwen3-8B} and also for \texttt{Llama3.1-8B}, training exhibits a \emph{Fast-Fit / Sustained Slow-Gain Reversal}: high-likelihood supervision retains an early advantage, but low-likelihood supervision produces substantially larger delayed gains and can eventually overtake high-likelihood training across a broad range of later checkpoints. These three regimes suggest a gradual capacity-dependent transition from failure to absorb off-distribution supervision, through unstable access to its additional knowledge, to stable and sustained exploitation of that knowledge, providing a more fine-grained characterization of the interaction between model capacity, data likelihood, and training duration.

\newcommand{\best}[1]{\cellcolor{bestyellow}\textbf{#1}}
\newcommand{\firstckpt}[1]{\cellcolor{firstcyan}\textbf{#1}}

\definecolor{bestyellow}{RGB}{255,242,204}
\definecolor{firstcyan}{RGB}{217,238,255}
\definecolor{bothgreen}{RGB}{226,239,218}

\newcommand{\bothmark}[1]{\cellcolor{bothgreen}\textbf{#1}}

\begin{table*}[t]

\centering
\small
\setlength{\tabcolsep}{4pt}
\renewcommand{\arraystretch}{1.10}

\caption{
\textbf{Qwen-2.5-1.5B results across training checkpoints under high- and low-likelihood data selection.}
\ \legendbox{firstcyan}\ \textbf{Cyan} marks the best value within the early stage (epochs 0--1),
\ \legendbox{bestyellow}\ \textbf{yellow} marks the best value across all checkpoints within each setting, and
\ \legendbox{bothgreen}\ \textbf{green} marks cells that satisfy both.
Dashed rules separate the earliest stage from later training.
}

\label{tab:qwen25_15b_high_low}

\resizebox{\textwidth}{!}{%

\begin{tabular}{llrrrrrrrrrrrr}

\toprule

\textbf{Setting} &
\textbf{Epoch} &
\textbf{AIME24} &
\textbf{AMC} &
\textbf{CNM} &
\textbf{GK} &
\textbf{GPQA} &
\textbf{GSM} &
\textbf{KY} &
\textbf{MATH} &
\textbf{MNV} &
\textbf{OB} &
\textbf{W.\ Avg} &
\textbf{Avg} \\

\midrule

\multirow[c]{21}{*}{\textbf{High}}
& 0    & 0 & 0 & 0 & 0 & 3.03 & 0 & 0 & 1.4 & 1.47 & 0.44 & 0.88 & 0.58 \\

& 0.25 & 0 & 22.5 & \bothmark{6.67} & 2.53 & 27.27 & 16.67 & 15.58 & 52.6 & 19.49 & 20.59 & 26.85 & 21.26 \\

& 0.5  & 0 & \firstckpt{27.5} & \bothmark{6.67} & \firstckpt{10.13} & 26.26 & \firstckpt{22.38} & \firstckpt{19.1} & 52.2 & 22.06 & 19.7 & 27.9 & 23.27 \\

& 0.75 & \firstckpt{3.33} & 25 & \bothmark{6.67} & 7.59 & \firstckpt{28.79} & 21.43 & 16.58 & \firstckpt{54.6} & 18.01 & \firstckpt{21.78} & \firstckpt{28.52} & 23.62 \\

& 1    & \firstckpt{3.33} & \firstckpt{27.5} & \bothmark{6.67} & \firstckpt{10.13} & 22.73 & 20.48 & 17.59 & 53.4 & \bothmark{24.63} & 19.26 & 27.99 & \bothmark{24.16} \\

\cdashline{2-14}

& 1.25 & \best{6.67} & 25 & 0 & \best{11.39} & 23.74 & \best{22.86} & 18.09 & 54 & 22.43 & 18.07 & 27.6 & 22.93 \\

& 1.5  & \best{6.67} & 22.5 & 0 & 6.33 & 26.77 & 22.38 & 16.58 & 55.8 & 23.53 & 18.67 & 28.21 & 22.84 \\

& 1.75 & 0 & 27.5 & 0 & 5.06 & 21.21 & 21.9 & \best{19.6} & 54.6 & 20.96 & 17.33 & 26.85 & 21.47 \\

& 2    & 0 & 30 & 0 & 1.27 & 25.25 & 20.48 & 14.57 & \best{57.6} & 22.43 & 19.56 & 28.34 & 23.01 \\

& 2.25 & 3.33 & 27.5 & 0 & 3.8 & 26.77 & 20.00 & 15.58 & 56.2 & 20.59 & 20.74 & 28.34 & 22.96 \\

& 2.5  & 3.33 & 25 & 0 & 7.59 & 25.25 & \best{22.86} & 17.09 & 55.8 & 22.79 & 21.93 & \best{29.35} & 24.15 \\

& 2.75 & 3.33 & 22.5 & \best{6.67} & 2.53 & 25.25 & 17.62 & 15.08 & 55.4 & 19.49 & 21.04 & 27.81 & 22.99 \\

& 3    & 3.33 & 27.5 & 3.33 & 2.53 & 24.75 & 14.76 & 13.57 & 55.4 & 22.06 & 19.7 & 27.51 & 23.54 \\

& 3.25 & \best{6.67} & 25 & 0 & 3.8 & \best{29.29} & 13.33 & 16.08 & 55 & 21.32 & 20.89 & 27.99 & 23.22 \\

& 3.5  & 3.33 & \best{35} & 0 & 3.8 & 23.74 & 14.29 & 13.57 & 57.2 & 22.43 & 20.15 & 28.08 & 24.14 \\

& 3.75 & \best{6.67} & 22.5 & 0 & 3.8 & 22.73 & 15.24 & 14.57 & 56.8 & 20.59 & 20.74 & 27.99 & 23.78 \\

& 4    & 0 & 22.5 & 3.33 & 2.53 & 24.24 & 12.86 & 13.57 & 54.8 & 22.79 & 21.04 & 27.68 & 23.42 \\

& 4.25 & 0 & 22.5 & 3.33 & 2.53 & 25.76 & 12.38 & 14.07 & 55.4 & 22.06 & 20.3 & 27.55 & 23.12 \\

& 4.5  & 0 & 22.5 & 3.33 & 2.53 & 23.74 & 13.81 & 13.57 & 56 & 21.32 & \best{22.52} & 28.12 & 23.03 \\

& 4.75 & 0 & 22.5 & 3.33 & 2.53 & 24.24 & 14.29 & 18.09 & 54.2 & 23.53 & 21.19 & 28.03 & 23.26 \\

& 5    & 0 & 17.5 & 3.33 & 2.53 & 22.73 & 12.86 & 12.06 & 52.8 & 22.79 & 20.00 & 26.41 & 21.69 \\

\midrule

\multirow[c]{21}{*}{\textbf{Low}}
& 0    & \firstckpt{0} & 0 & 0 & 0 & 3.03 & 0 & 0 & 1.4 & 1.47 & 0.44 & 0.88 & 0.58 \\

& 0.25 & \firstckpt{0} & 12.5 & 0 & 1.27 & \bothmark{16.67} & 10.48 & 5.03 & 37 & \firstckpt{13.24} & 11.26 & 16.69 & 12.13 \\

& 0.5  & \firstckpt{0} & \firstckpt{17.5} & 0 & 3.8 & 11.11 & 8.57 & 5.03 & 41.2 & 9.93 & 10.96 & 16.95 & 13.46 \\

& 0.75 & \firstckpt{0} & 12.5 & 0 & \firstckpt{6.33} & 7.07 & 11.43 & 2.51 & 40.4 & 12.5 & 10.96 & 16.73 & 12.88 \\

& 1    & \firstckpt{0} & \firstckpt{17.5} & \firstckpt{3.33} & 3.8 & 11.11 & \bothmark{13.33} & \firstckpt{8.04} & \firstckpt{41.8} & 12.5 & \firstckpt{13.04} & \firstckpt{18.7} & \firstckpt{14.77} \\

\cdashline{2-14}

& 1.25 & \best{3.33} & 17.5 & 0 & 6.33 & 14.65 & 10.48 & 5.53 & 41.2 & 12.87 & 13.19 & 18.57 & 14.82 \\

& 1.5  & \best{3.33} & 12.5 & 3.33 & 6.33 & 10.1 & 8.1 & 4.52 & 41.8 & 13.6 & 11.85 & 17.61 & 13.77 \\

& 1.75 & \best{3.33} & 15 & 0 & 7.59 & 13.13 & 8.1 & 6.03 & 41.6 & 12.5 & \best{14.96} & 18.83 & 14.57 \\

& 2    & \best{3.33} & 20.00 & 0 & 5.06 & 9.09 & 8.57 & 9.05 & 40.2 & \best{14.71} & 12.89 & 18.22 & 14.99 \\

& 2.25 & \best{3.33} & \best{22.5} & 0 & \best{8.86} & 11.11 & 11.43 & 8.04 & 39.8 & 9.93 & 14.52 & 18.79 & \best{16.5} \\

& 2.5  & 0 & 17.5 & 0 & 5.06 & 11.11 & 10.95 & \best{10.55} & 41.6 & 10.29 & 12.44 & 18.44 & 15.23 \\

& 2.75 & 0 & 15 & 0 & 7.59 & 13.13 & 9.05 & 5.03 & 43.2 & 12.5 & 13.19 & 18.7 & 14.61 \\

& 3    & 0 & \best{22.5} & 0 & 7.59 & 10.1 & 10.95 & 4.52 & 43.8 & 14.34 & 12.15 & 18.83 & 15.63 \\

& 3.25 & 0 & 10 & \best{6.67} & 3.8 & 13.64 & 10.48 & 7.04 & \best{44.4} & 13.97 & 13.33 & \best{19.54} & 15.57 \\

& 3.5  & 0 & 17.5 & 0 & 3.8 & 12.12 & 10.48 & 7.04 & 39.8 & 10.66 & 12 & 17.74 & 15.04 \\

& 3.75 & 0 & 12.5 & 0 & 5.06 & 10.61 & 12.38 & 7.54 & 41.8 & 10.29 & 13.63 & 18.57 & 14.71 \\

& 4    & 0 & \best{22.5} & 0 & 5.06 & 11.11 & 9.05 & 6.53 & 43.8 & 12.13 & 13.33 & 19.1 & 16.14 \\

& 4.25 & 0 & 20.00 & 0 & 7.59 & 12.63 & 10.48 & 5.03 & 43.2 & 11.4 & 12.89 & 18.83 & 15.75 \\

& 4.5  & 0 & 15 & 0 & 6.33 & 11.62 & 12.86 & 6.03 & 42.8 & 13.24 & 12.89 & 19.01 & 15.34 \\

& 4.75 & 0 & 15 & 0 & 6.33 & 12.63 & \best{13.33} & 7.04 & 43 & 12.5 & 12.89 & 19.27 & 15.88 \\

& 5    & \best{3.33} & 15 & \best{6.67} & 3.8 & 9.09 & 11.43 & 4.52 & \best{44.4} & 13.6 & 12.74 & 18.97 & 15.87 \\

\bottomrule

\end{tabular}%

}

\end{table*}

\begin{table*}[t]

\centering
\small
\setlength{\tabcolsep}{4pt}
\renewcommand{\arraystretch}{1.10}

\caption{
\textbf{Qwen-2.5-3B-Instruct results across training checkpoints under high- and low-likelihood data selection.}
\ \legendbox{firstcyan}\ \textbf{Cyan} marks the best value within the early stage (epochs 0.25--1),
\ \legendbox{bestyellow}\ \textbf{yellow} marks the best value across epochs 0.25--5 within each setting, and
\ \legendbox{bothgreen}\ \textbf{green} marks cells that satisfy both.
Dashed rules separate the earliest stage from later training.
}

\label{tab:qwen25_3b_high_low}

\resizebox{\textwidth}{!}{%

\begin{tabular}{llrrrrrrrrrrrr}

\toprule

\textbf{Setting} &
\textbf{Epoch} &
\textbf{AIME24} &
\textbf{AMC} &
\textbf{CNM} &
\textbf{GK} &
\textbf{GPQA} &
\textbf{GSM} &
\textbf{KY} &
\textbf{MATH} &
\textbf{MNV} &
\textbf{OB} &
\textbf{W.\ Avg} &
\textbf{Avg} \\

\midrule

\multirow[c]{20}{*}{\textbf{High}}
& 0.25 & \bothmark{10} & \firstckpt{47.5} & \firstckpt{10} & \bothmark{25.32} & 27.27 & \firstckpt{19.52} & \firstckpt{24.12} & 64.6 & 29.78 & 26.37 & \firstckpt{34.87} & \firstckpt{30.59} \\

& 0.5  & \bothmark{10} & 32.5 & 0 & 12.66 & 27.78 & 17.62 & 23.12 & 63.2 & \firstckpt{32.35} & 25.04 & 33.46 & 27.11 \\

& 0.75 & 6.67 & 25 & 0 & 13.92 & 28.79 & 14.29 & 16.58 & \firstckpt{66.4} & 28.68 & \firstckpt{26.52} & 33.46 & 26.44 \\

& 1    & 3.33 & 40 & \firstckpt{10} & 12.66 & \bothmark{34.85} & 16.67 & \firstckpt{24.12} & 62 & 29.78 & \firstckpt{26.52} & 34.25 & 29.08 \\

\cdashline{2-14}

& 1.25 & 6.67 & 45 & \best{16.67} & 17.72 & 30.81 & 15.71 & 20.6 & 65 & 30.51 & 29.63 & 35.79 & 31.67 \\

& 1.5  & 3.33 & 37.5 & 6.67 & 12.66 & 30.3 & 11.9 & 19.1 & 63 & 29.41 & 28.74 & 33.9 & 28.24 \\

& 1.75 & 6.67 & 30 & 10 & \best{25.32} & 29.29 & \best{22.38} & \best{26.63} & \best{67.2} & 31.25 & \best{30.52} & \best{37.58} & 31.93 \\

& 2    & 6.67 & 27.5 & 6.67 & 13.92 & 30.3 & 13.81 & 20.6 & 64 & 32.35 & 29.48 & 34.95 & 28.66 \\

& 2.25 & 6.67 & 40 & 13.33 & 21.52 & 33.84 & 17.62 & 17.59 & 64.6 & 29.78 & 28.15 & 35.39 & 31.37 \\

& 2.5  & 6.67 & 45 & 13.33 & 18.99 & 31.31 & 17.14 & 17.09 & 65 & 29.04 & 29.19 & 35.52 & 31.89 \\

& 2.75 & 6.67 & 45 & 13.33 & 12.66 & 29.8 & 16.19 & 24.62 & 65.4 & 29.04 & 29.33 & 35.87 & 31.82 \\

& 3    & \best{10} & 45 & 10 & 16.46 & 30.3 & 17.14 & 22.61 & 63.2 & 31.99 & 29.48 & 35.92 & 32.38 \\

& 3.25 & 3.33 & 50 & 10 & 15.19 & 29.29 & 15.24 & 15.08 & 63.2 & 29.04 & 28.74 & 34.47 & 31.19 \\

& 3.5  & 6.67 & 37.5 & \best{16.67} & 11.39 & 26.77 & 10 & 22.11 & 63.2 & 29.04 & 26.37 & 33.42 & 30.16 \\

& 3.75 & 3.33 & 35 & 13.33 & 10.13 & 28.79 & 14.76 & 14.57 & 64.4 & 29.78 & 28.44 & 34.25 & 29.87 \\

& 4    & \best{10} & 47.5 & 13.33 & 13.92 & 27.27 & 12.86 & 20.1 & 66.8 & 31.25 & 29.19 & 35.83 & 32.75 \\

& 4.25 & 6.67 & 40 & 6.67 & 8.86 & 28.28 & 11.43 & 22.11 & 65 & 29.04 & 29.33 & 34.91 & 30.49 \\

& 4.5  & 6.67 & 45 & \best{16.67} & 7.59 & 27.27 & 13.33 & 21.11 & 65.4 & \best{32.72} & 29.33 & 35.61 & 32.1 \\

& 4.75 & \best{10} & \best{55} & 10 & 13.92 & 31.31 & 12.38 & 20.1 & 64.6 & 31.62 & 28.3 & 35.57 & \best{33.38} \\

& 5    & 6.67 & 42.5 & 10 & 6.33 & 28.28 & 12.86 & 20.6 & 66.4 & 32.35 & 28.15 & 35.3 & 31.29 \\

\midrule

\multirow[c]{20}{*}{\textbf{Low}}
& 0.25 & 0 & 27.5 & 0 & 7.59 & 13.13 & 11.9 & \firstckpt{7.04} & 48.8 & 22.79 & 18.81 & 23.57 & 18.51 \\

& 0.5  & \firstckpt{3.33} & \bothmark{32.5} & \firstckpt{3.33} & 5.06 & \firstckpt{18.18} & 10.95 & 6.53 & 50 & 24.26 & 19.11 & 24.49 & \firstckpt{19.93} \\

& 0.75 & 0 & 27.5 & \firstckpt{3.33} & \firstckpt{10.13} & 12.63 & \firstckpt{14.76} & 6.03 & 50.4 & 18.75 & 20.15 & 24 & 18.7 \\

& 1    & 0 & 30 & 0 & 8.86 & 15.66 & 12.38 & 3.02 & \firstckpt{52.2} & \firstckpt{25.37} & \firstckpt{20.59} & \firstckpt{25.19} & 19.64 \\

\cdashline{2-14}

& 1.25 & 0 & 27.5 & 3.33 & 13.92 & 17.17 & 14.76 & 8.04 & 55 & 25 & 21.48 & 26.85 & 20.75 \\

& 1.5  & 0 & 25 & 0 & 8.86 & 19.7 & 15.71 & 7.04 & 52.4 & 25 & 21.63 & 26.37 & 20.12 \\

& 1.75 & 3.33 & 25 & 3.33 & 12.66 & 15.15 & 13.81 & 5.03 & 55.2 & 23.9 & 22.81 & 26.76 & 20.93 \\

& 2    & 0 & 30 & 3.33 & 13.92 & 18.69 & 14.76 & 8.54 & 54.2 & 22.06 & 21.93 & 26.81 & 21.4 \\

& 2.25 & 0 & 25 & 0 & 15.19 & 15.66 & 15.71 & 8.04 & 51.8 & 22.06 & 22.07 & 26.19 & 21.05 \\

& 2.5  & 0 & 22.5 & 0 & 16.46 & 16.16 & 13.33 & 8.04 & 55.6 & 23.16 & 22.96 & 27.07 & 20.57 \\

& 2.75 & 3.33 & 27.5 & 0 & \best{24.05} & 16.16 & 14.76 & 10.05 & 56.2 & \best{25.74} & 23.11 & 28.47 & 23.54 \\

& 3    & 3.33 & 30 & \best{10} & 12.66 & 17.17 & \best{19.52} & 7.54 & 56 & 25.37 & 24 & 28.6 & 23.42 \\

& 3.25 & 0 & 27.5 & 3.33 & 16.46 & 18.18 & 15.24 & 8.04 & 55.8 & 20.96 & 22.81 & 27.51 & 22.39 \\

& 3.5  & 3.33 & 27.5 & 0 & 16.46 & 23.74 & 17.62 & 10.55 & 56.4 & 23.53 & 21.33 & 28.43 & 23.5 \\

& 3.75 & 3.33 & 25 & 6.67 & 18.99 & 22.73 & 17.14 & \best{11.56} & 56.8 & 24.63 & 22.07 & 29.04 & \best{24.63} \\

& 4    & 0 & 27.5 & 6.67 & 15.19 & \best{27.27} & 18.1 & 10.05 & 56.6 & 22.43 & 23.11 & \best{29.3} & \best{24.63} \\

& 4.25 & 3.33 & 20 & 3.33 & 13.92 & 22.73 & 16.19 & 8.04 & 58 & 23.53 & 23.26 & 28.82 & 23.12 \\

& 4.5  & \best{6.67} & 27.5 & 0 & 11.39 & 24.24 & 15.71 & 10.55 & \best{58.6} & 22.43 & 22.96 & 29.04 & 23.64 \\

& 4.75 & 3.33 & 25 & 0 & 13.92 & 25.76 & 17.14 & 8.04 & 57.8 & 20.59 & 22.52 & 28.56 & 23.1 \\

& 5    & 3.33 & 27.5 & 3.33 & 15.19 & 21.21 & 18.1 & 8.54 & 56.4 & 18.75 & \best{24.3} & 28.43 & 23.33 \\

\bottomrule

\end{tabular}%

}

\end{table*}

\begin{table*}[t]

\centering
\small
\setlength{\tabcolsep}{4pt}
\renewcommand{\arraystretch}{1.10}

\caption{
\textbf{Qwen-2.5-7B results across training checkpoints under high- and low-likelihood data selection.}
\ \legendbox{firstcyan}\ \textbf{Cyan} marks the best value within the early stage (epochs 0.25--1),
\ \legendbox{bestyellow}\ \textbf{yellow} marks the best value across epochs 0.25--5 within each setting, and
\ \legendbox{bothgreen}\ \textbf{green} marks cells that satisfy both.
Dashed rules separate the earliest stage from later training.
}

\label{tab:qwen25_7b_high_low}

\resizebox{\textwidth}{!}{%

\begin{tabular}{llrrrrrrrrrrrr}

\toprule

\textbf{Setting} &
\textbf{Epoch} &
\textbf{AIME} &
\textbf{AMC} &
\textbf{CNM} &
\textbf{GK} &
\textbf{GPQA} &
\textbf{GSM} &
\textbf{KY} &
\textbf{MATH} &
\textbf{MNV} &
\textbf{OB} &
\textbf{W.\ Avg} &
\textbf{Avg} \\

\midrule

\multirow[c]{20}{*}{\textbf{High}}
& 0.25 & 10.00 & \firstckpt{65.00} & 26.67 & 49.37 & 37.37 & 45.71 & \firstckpt{46.23} & 77.80 & \firstckpt{43.01} & 39.70 & 49.80 & \firstckpt{44.09} \\

& 0.5  & \firstckpt{13.33} & 62.50 & \firstckpt{30.00} & \bothmark{55.70} & 33.84 & \firstckpt{48.10} & 35.68 & 78.00 & 42.65 & 39.70 & 49.04 & 43.95 \\

& 0.75 & \firstckpt{13.33} & 57.50 & 13.33 & 51.90 & \bothmark{39.39} & 42.86 & 39.70 & 75.80 & 42.65 & 39.11 & 48.28 & 41.56 \\

& 1    & 10.00 & 62.50 & 20.00 & 49.37 & 33.84 & 47.14 & 38.69 & \firstckpt{78.80} & \firstckpt{43.01} & \firstckpt{42.37} & \firstckpt{49.84} & 42.57 \\

\cdashline{2-14}

& 1.25 & 16.67 & 62.50 & 23.33 & 49.37 & 32.83 & 45.24 & 39.70 & 77.80 & 41.54 & 41.93 & 49.26 & 43.09 \\

& 1.5  & 13.33 & 52.50 & 16.67 & 50.63 & 31.31 & 47.62 & 39.20 & 75.20 & 40.07 & 40.15 & 47.74 & 40.67 \\

& 1.75 & 10.00 & 62.50 & 23.33 & 54.43 & 26.77 & 47.62 & 40.20 & 78.60 & 41.54 & 41.78 & 49.22 & 42.68 \\

& 2    & 13.33 & 55.00 & 26.67 & 50.63 & 30.30 & 47.14 & 41.71 & 77.40 & 43.01 & 42.22 & 49.48 & 42.74 \\

& 2.25 & \best{20.00} & 50.00 & 26.67 & 51.90 & 32.83 & \best{50.95} & 41.21 & 78.60 & \best{44.12} & 41.04 & 50.11 & 43.73 \\

& 2.5  & 10.00 & 57.50 & 20.00 & 50.63 & 35.35 & 46.19 & 42.21 & 78.60 & 43.38 & 42.81 & 50.29 & 42.67 \\

& 2.75 & 10.00 & \best{67.50} & 33.33 & 50.63 & 38.38 & 47.14 & 41.21 & 77.40 & 43.01 & 41.48 & 50.20 & 45.01 \\

& 3    & 13.33 & 65.00 & \best{36.67} & 54.43 & 34.85 & 44.76 & 42.21 & 78.60 & 42.65 & 41.63 & 50.20 & \best{45.41} \\

& 3.25 & 13.33 & 55.00 & 30.00 & 54.43 & 33.33 & 42.86 & 42.71 & 78.60 & 41.54 & 41.04 & 49.35 & 43.28 \\

& 3.5  & 13.33 & 57.50 & 30.00 & 54.43 & 36.36 & 43.33 & \best{46.73} & 78.00 & 43.01 & \best{43.26} & 50.78 & 44.60 \\

& 3.75 & 16.67 & 57.50 & 30.00 & 51.90 & 36.36 & 48.10 & 44.72 & 77.80 & 41.18 & 42.37 & 50.47 & 44.66 \\

& 4    & 10.00 & 52.50 & 33.33 & 53.16 & 34.85 & 44.29 & 44.22 & 78.00 & 42.65 & 42.07 & 49.98 & 43.51 \\

& 4.25 & 13.33 & 50.00 & 33.33 & 54.43 & 32.83 & 44.76 & 45.23 & 78.20 & 42.65 & 41.63 & 49.89 & 43.64 \\

& 4.5  & 16.67 & 47.50 & 30.00 & 53.16 & 35.86 & 45.24 & 43.72 & 77.00 & 42.65 & 42.22 & 49.89 & 43.40 \\

& 4.75 & 16.67 & 47.50 & 30.00 & \best{55.70} & 36.87 & 46.19 & 40.70 & 78.40 & 43.38 & 41.78 & 50.16 & 43.72 \\

& 5    & 13.33 & 50.00 & 23.33 & 53.16 & 38.89 & 47.14 & 46.23 & \best{79.00} & 41.54 & 42.52 & \best{50.87} & 43.52 \\

\midrule

\multirow[c]{20}{*}{\textbf{Low}}
& 0.25 & \firstckpt{13.33} & 45.00 & 16.67 & 45.57 & 23.74 & 40.95 & 30.65 & 68.00 & 30.51 & 34.22 & 40.80 & 34.86 \\

& 0.5  & \firstckpt{13.33} & 45.00 & 6.67 & 55.70 & 23.74 & 42.86 & \firstckpt{31.16} & 71.40 & \firstckpt{33.82} & 37.48 & \firstckpt{43.39} & 36.12 \\

& 0.75 & 6.67 & \firstckpt{47.50} & 16.67 & \firstckpt{56.96} & \firstckpt{27.78} & 41.90 & 27.14 & 69.60 & 31.25 & 37.19 & 42.63 & 36.26 \\

& 1    & 10.00 & \firstckpt{47.50} & \firstckpt{20.00} & \firstckpt{56.96} & 23.74 & \firstckpt{44.29} & 28.14 & \firstckpt{71.80} & 30.15 & \firstckpt{37.93} & 43.26 & \firstckpt{37.05} \\

\cdashline{2-14}

& 1.25 & 16.67 & 52.50 & 30.00 & \best{62.03} & 33.33 & 46.67 & \best{37.69} & 74.40 & 37.13 & 40.30 & 47.83 & 43.07 \\

& 1.5  & 13.33 & 45.00 & 13.33 & 56.96 & 29.29 & 45.24 & 34.17 & 72.00 & 30.88 & 36.74 & 44.07 & 37.70 \\

& 1.75 & 16.67 & 62.50 & 26.67 & 56.96 & 25.25 & 49.05 & 34.17 & 74.80 & 33.82 & 40.00 & 46.57 & 41.99 \\

& 2    & 20.00 & 50.00 & 20.00 & 58.23 & 38.89 & 49.05 & 32.16 & 78.60 & 34.93 & 44.00 & 49.57 & 42.59 \\

& 2.25 & 20.00 & 62.50 & 23.33 & \best{62.03} & 38.38 & 51.90 & 33.17 & 78.20 & \best{37.87} & 42.81 & 50.20 & 45.02 \\

& 2.5  & 20.00 & 62.50 & 23.33 & 55.70 & 35.86 & \best{55.71} & 32.66 & 78.80 & 36.03 & 42.96 & 50.02 & 44.36 \\

& 2.75 & 13.33 & 60.00 & 20.00 & 55.70 & 35.35 & 49.52 & 32.66 & 77.60 & 33.82 & 40.44 & 47.92 & 41.84 \\

& 3    & 20.00 & 50.00 & 20.00 & 50.63 & 40.91 & 54.29 & 35.18 & 77.80 & 34.19 & 40.74 & 48.99 & 42.37 \\

& 3.25 & 16.67 & 57.50 & 23.33 & 50.63 & 35.86 & 53.81 & 33.67 & 79.60 & 36.03 & 42.52 & 49.66 & 42.96 \\

& 3.5  & 20.00 & 55.00 & 30.00 & 58.23 & 38.89 & 52.86 & 33.67 & 78.40 & 35.66 & 42.07 & 49.75 & 44.48 \\

& 3.75 & 20.00 & 47.50 & 30.00 & 54.43 & 39.90 & 55.24 & 34.67 & \best{80.60} & 36.40 & 43.11 & 50.78 & 44.18 \\

& 4    & 20.00 & 62.50 & \best{33.33} & 58.23 & 37.37 & 51.90 & 33.67 & 79.20 & 35.29 & 43.26 & 50.20 & 45.48 \\

& 4.25 & 16.67 & 60.00 & 20.00 & 56.96 & 39.90 & 50.48 & 34.17 & 79.20 & 34.56 & 44.15 & 50.20 & 43.61 \\

& 4.5  & \best{26.67} & 60.00 & 30.00 & 53.16 & \best{44.95} & 52.38 & 32.66 & 79.20 & 34.19 & 43.11 & 50.47 & 45.63 \\

& 4.75 & 23.33 & \best{65.00} & 26.67 & 55.70 & 43.43 & 53.81 & 36.18 & 80.00 & 33.46 & \best{44.30} & \best{51.32} & \best{46.19} \\

& 5    & 13.33 & 55.00 & 23.33 & 51.90 & 42.42 & 50.95 & 36.68 & 79.40 & 34.93 & 44.00 & 50.47 & 43.20 \\

\bottomrule

\end{tabular}%

}

\end{table*}

\begin{table*}[t]

\centering
\small
\setlength{\tabcolsep}{4pt}
\renewcommand{\arraystretch}{1.10}

\caption{
\textbf{Qwen3-8B results across training epochs under high- and low-likelihood data selection.}
\ \legendbox{firstcyan}\ \textbf{Cyan} marks the best value within the early stage (epochs 0.25--1),
\ \legendbox{bestyellow}\ \textbf{yellow} marks the best value across epochs 0.25--5 within each setting, and
\ \legendbox{bothgreen}\ \textbf{green} marks cells that satisfy both.
Dashed rules separate the earliest stage from later training.
}

\label{tab:qwen3_8b_high_low}

\resizebox{\textwidth}{!}{%

\begin{tabular}{llrrrrrrrrrrrr}

\toprule

\textbf{Setting} &
\textbf{Epoch} &
\textbf{AIME24} &
\textbf{AMC} &
\textbf{CNM} &
\textbf{GK} &
\textbf{GPQA} &
\textbf{GSM} &
\textbf{KY} &
\textbf{MATH} &
\textbf{MNV} &
\textbf{OB} &
\textbf{W.\ Avg} &
\textbf{Avg} \\

\midrule

\multirow[c]{20}{*}{\textbf{High}}
& 0.25 & \firstckpt{36.67} & 80 & 50 & 64.56 & \bothmark{78.79} & \firstckpt{68.10} & \bothmark{71.86} & \firstckpt{91.6} & 58.09 & 61.93 & \bothmark{70.98} & \firstckpt{66.16} \\

& 0.5  & 26.67 & \firstckpt{87.5} & 46.67 & 68.35 & 72.73 & 66.67 & 67.34 & 90.0 & \bothmark{58.46} & 63.56 & 70.17 & 64.79 \\

& 0.75 & 33.33 & 85 & 46.67 & 69.62 & 75.25 & \firstckpt{68.10} & 67.84 & 90.0 & 55.51 & \firstckpt{64.30} & 70.53 & 65.56 \\

& 1    & 26.67 & 82.5 & \bothmark{56.67} & \bothmark{74.68} & 76.77 & 67.62 & 62.31 & 91.0 & 56.99 & 62.67 & 70.22 & 65.79 \\

\cdashline{2-14}

& 1.25 & 36.67 & 85 & 43.33 & 73.42 & 72.73 & 68.57 & 63.82 & \best{92.0} & 55.88 & 62.52 & 70.09 & 65.39 \\

& 1.5  & 33.33 & 87.5 & 50.00 & 73.42 & 76.77 & 70.48 & 63.32 & 91.2 & 57.72 & 62.67 & 70.76 & \best{66.64} \\

& 1.75 & 26.67 & 85 & 43.33 & 69.62 & 74.75 & \best{72.86} & 67.84 & 90.0 & 55.88 & 63.70 & 70.67 & 64.97 \\

& 2    & 33.33 & 82.5 & 53.33 & 70.89 & 70.71 & 70.00 & 67.34 & 91.6 & 54.41 & \best{65.19} & 70.85 & 65.93 \\

& 2.25 & 33.33 & 80 & 53.33 & 68.35 & 72.73 & 65.24 & 65.33 & 91.8 & 55.88 & 63.70 & 70.04 & 64.97 \\

& 2.5  & 36.67 & 82.5 & 43.33 & 67.09 & 77.27 & 65.24 & 65.83 & 90.8 & 55.51 & 64.15 & 70.26 & 64.84 \\

& 2.75 & 30.00 & 82.5 & 53.33 & 68.35 & 76.26 & 67.14 & 63.32 & 91.4 & 55.88 & 62.96 & 70.04 & 65.12 \\

& 3    & 26.67 & \best{90.0} & 43.33 & 69.62 & 72.22 & 68.10 & 59.30 & 91.8 & 55.15 & 62.52 & 69.28 & 63.87 \\

& 3.25 & \best{40.00} & 85.0 & 43.33 & 67.09 & 75.25 & 69.05 & 62.81 & 91.4 & 56.62 & 64.59 & 70.67 & 65.51 \\

& 3.5  & 30.00 & \best{90.0} & 40.00 & 67.09 & 76.77 & 67.14 & 64.32 & 90.4 & 56.62 & 64.74 & 70.49 & 64.71 \\

& 3.75 & 30.00 & 82.5 & 50.00 & 65.82 & 71.21 & 67.14 & 66.33 & 91.8 & 55.51 & 62.96 & 69.77 & 64.33 \\

& 4    & 26.67 & \best{90.0} & 43.33 & 65.82 & 75.76 & 68.57 & 66.33 & 90.6 & 55.15 & 62.81 & 69.95 & 64.50 \\

& 4.25 & 33.33 & 85.0 & 43.33 & 67.09 & 75.25 & 67.62 & 65.33 & 90.6 & 54.41 & 63.85 & 70.00 & 64.58 \\

& 4.5  & 26.67 & 82.5 & 43.33 & 65.82 & 77.27 & 66.19 & 62.81 & 90.6 & 53.31 & 63.11 & 69.28 & 63.16 \\

& 4.75 & 26.67 & 87.5 & 43.33 & 68.35 & 75.25 & 68.10 & 62.81 & 91.4 & 55.88 & 64.00 & 70.22 & 64.33 \\

& 5    & 36.67 & 85.0 & 40.00 & 67.09 & 73.23 & 68.10 & 64.82 & 90.4 & 55.51 & 62.52 & 69.50 & 64.33 \\

\midrule

\multirow[c]{20}{*}{\textbf{Low}}
& 0.25 & 20.00 & 77.5 & 23.33 & 51.9 & \firstckpt{75.25} & 63.81 & 64.82 & 88.20 & 58.19 & 54.67 & 66.65 & 57.72 \\

& 0.5  & \firstckpt{26.67} & \firstckpt{80.00} & 40 & 62.03 & 71.72 & \firstckpt{65.24} & \firstckpt{65.33} & \firstckpt{90.00} & \firstckpt{58.46} & 57.19 & \firstckpt{74.65} & \firstckpt{62.62} \\

& 0.75 & 13.33 & 60.00 & \firstckpt{50.00} & \firstckpt{67.09} & 71.21 & 63.33 & 57.29 & 87.20 & 58.09 & 56.00 & 71.54 & 58.68 \\

& 1    & 20.00 & 67.5 & \firstckpt{50} & 63.29 & 73.23 & 64.76 & 62.81 & 89.20 & 57.35 & \firstckpt{58.96} & 73.87 & 61.35 \\

\cdashline{2-14}

& 1.25 & 26.67 & 77.5 & 50 & 67.09 & 75.76 & 66.67 & 58.29 & 90.20 & 55.15 & 58.96 & 74.96 & 64.02 \\

& 1.5  & 26.67 & 82.5 & 43.33 & 67.09 & 76.77 & 68.10 & 66.83 & 90.60 & 59.93 & 60.59 & 77.14 & 65.49 \\

& 1.75 & 26.67 & 75 & 40 & 69.62 & 78.28 & 66.19 & 63.32 & 90.60 & 58.09 & 62.52 & 76.05 & 63.71 \\

& 2    & 30 & 72.5 & 53.33 & 68.35 & 76.26 & 68.57 & 63.32 & 92.20 & 59.93 & 62.96 & 77.14 & 65.69 \\

& 2.25 & 33.33 & 82.5 & 53.33 & \best{70.89} & 75.76 & 70.48 & 65.33 & 91.20 & 57.35 & 62.81 & 77.68 & 67.85 \\

& 2.5  & 40.00 & 82.5 & \best{63.33} & 69.62 & 76.77 & 71.43 & 67.34 & 90.4 & 58.82 & 62.37 & 72.82 & 69.31 \\

& 2.75 & 40.00 & 82.5 & 53.33 & 65.82 & 76.26 & 69.52 & 64.82 & 91.8 & 58.46 & 61.78 & 72.16 & 67.32 \\

& 3    & \best{43.33} & 85 & 60 & 68.35 & 76.26 & 71.43 & 61.81 & 90.8 & 60.29 & 60.30 & 71.60 & 68.59 \\

& 3.25 & 36.67 & 82.5 & 53.33 & 69.62 & 79.29 & 69.05 & 65.83 & 91.4 & 59.93 & 64.44 & 73.43 & 68.01 \\

& 3.5  & 40.00 & 87.5 & 60.00 & 67.09 & 77.27 & \best{72.38} & 64.32 & 92.4 & 58.46 & 62.37 & \best{78.77} & 70.12 \\

& 3.75 & \best{43.33} & 90 & \best{63.33} & 67.09 & 75.25 & 71.43 & 63.32 & 91.8 & 58.82 & 63.85 & 78.15 & \best{70.69} \\

& 4    & 36.67 & 82.5 & 53.33 & 68.35 & \best{79.80} & 70.48 & 65.33 & 91.2 & 58.09 & 61.48 & 78.23 & 68.46 \\

& 4.25 & 26.67 & 82.5 & 60 & 67.09 & 74.24 & 70.00 & \best{69.85} & 92.2 & 58.46 & 63.56 & 78.23 & 67.82 \\

& 4.5  & 36.67 & \best{92.5} & 56.67 & 65.82 & 76.77 & 70.95 & 63.82 & 91.0 & \best{61.40} & 63.56 & 77.76 & 69.27 \\

& 4.75 & 36.67 & 82.5 & 60.00 & 64.56 & 75.25 & 69.05 & 64.82 & 92.8 & 58.82 & \best{64.59} & 77.76 & 68.21 \\

& 5    & 30 & 85 & \best{63.33} & 67.09 & 78.28 & 68.10 & 67.34 & \best{93.0} & 59.19 & 62.07 & 78.69 & 69.02 \\

\bottomrule

\end{tabular}%

}

\end{table*}

\begin{table*}[t]

\centering
\small
\setlength{\tabcolsep}{4pt}
\renewcommand{\arraystretch}{1.10}

\caption{
\textbf{Llama-3.2-3B results across training epochs under high- and low-likelihood data selection.}
\ \legendbox{firstcyan}\ \textbf{Cyan} marks the best value within the early stage (epochs 0.25--1),
\ \legendbox{bestyellow}\ \textbf{yellow} marks the best value across epochs 0.25--5 within each setting, and
\ \legendbox{bothgreen}\ \textbf{green} marks cells that satisfy both.
Dashed rules separate the earliest stage from later training.
}

\label{tab:llama32_3b_high_low}

\resizebox{\textwidth}{!}{%

\begin{tabular}{llrrrrrrrrrrrr}

\toprule

\textbf{Setting} &
\textbf{Epoch} &
\textbf{AIME24} &
\textbf{AMC} &
\textbf{CNM} &
\textbf{GK} &
\textbf{GPQA} &
\textbf{GSM} &
\textbf{KY} &
\textbf{MATH} &
\textbf{MNV} &
\textbf{OB} &
\textbf{W.\ Avg} &
\textbf{Avg} \\

\midrule

\multirow[c]{20}{*}{\textbf{High}}
& 0.25 & 3.75 & \firstckpt{23.12} & \firstckpt{5.00} & \firstckpt{3.48} & 18.75 & 1.43 & 1.57 & 43.95 & 17.10 & 14.67 & 18.95 & 13.28 \\

& 0.5 & 5.83 & 17.50 & 4.17 & 1.90 & 20.08 & \firstckpt{2.86} & 1.38 & 41.52 & \firstckpt{18.38} & 13.85 & 18.41 & 12.75 \\

& 0.75 & \firstckpt{6.25} & 21.25 & 3.75 & 3.32 & 20.20 & \firstckpt{2.86} & 1.07 & \firstckpt{44.35} & 17.97 & \firstckpt{15.43} & \firstckpt{19.57} & \firstckpt{13.64} \\

& 1 & \firstckpt{6.25} & 18.44 & 2.92 & 3.01 & \firstckpt{23.42} & 2.74 & \firstckpt{3.14} & 44.12 & 17.78 & 14.33 & 19.55 & 13.62 \\

\cdashline{2-14}

& 1.25 & 7.50 & 26.25 & 4.17 & 3.16 & 20.83 & 2.08 & 1.76 & 46.65 & 18.98 & 16.15 & 20.58 & 14.75 \\

& 1.5 & 6.25 & 24.38 & 3.75 & 6.65 & 25.06 & 4.29 & 3.77 & 46.38 & 18.29 & 16.35 & 21.32 & 15.52 \\

& 1.75 & \best{8.33} & \best{29.06} & 2.92 & 5.06 & 22.60 & 4.82 & 2.58 & 46.95 & 20.04 & 16.46 & 21.47 & 15.88 \\

& 2 & 7.08 & 27.19 & 3.75 & 5.38 & 25.69 & \best{4.94} & 3.64 & 46.17 & 19.81 & 16.72 & 21.70 & 16.04 \\

& 2.25 & 5.00 & 23.75 & 2.92 & 5.85 & 25.06 & 3.63 & 3.71 & 45.85 & 18.93 & 16.54 & 21.20 & 15.12 \\

& 2.5 & 6.67 & 27.81 & 4.58 & \best{7.28} & \best{26.26} & 3.45 & 3.33 & 47.30 & 19.26 & 17.19 & 21.99 & \best{16.31} \\

& 2.75 & 7.92 & 28.12 & 1.25 & 5.54 & 25.82 & 4.29 & \best{4.65} & 47.73 & 20.54 & 16.15 & \best{22.00} & 16.20 \\

& 3 & \best{8.33} & 25.31 & \best{5.42} & 5.38 & 24.87 & 3.10 & 3.52 & 47.60 & 19.03 & 16.35 & 21.56 & 15.89 \\

& 3.25 & 4.58 & 27.50 & 5.00 & 3.48 & 21.78 & 1.73 & 2.45 & 46.40 & 19.81 & 17.22 & 21.06 & 14.99 \\

& 3.5 & 5.42 & 27.50 & 5.00 & 3.96 & 24.37 & 3.93 & 2.70 & 47.93 & 20.08 & 16.81 & 21.80 & 15.77 \\

& 3.75 & 7.08 & 26.25 & 2.92 & 3.64 & 24.56 & 1.85 & 3.64 & 46.40 & 19.44 & \best{17.56} & 21.47 & 15.33 \\

& 4 & 7.92 & 28.75 & 4.17 & 4.75 & 25.69 & 3.63 & 4.08 & 47.60 & 18.89 & 17.17 & 21.98 & 16.26 \\

& 4.25 & 7.92 & 26.25 & 5.00 & 3.48 & 23.74 & 2.50 & 2.95 & \best{48.80} & 20.27 & 16.93 & 21.88 & 15.78 \\

& 4.5 & 6.25 & 27.19 & 2.50 & 3.16 & 21.53 & 2.50 & 2.83 & 47.70 & \best{20.77} & 16.61 & 21.34 & 15.10 \\

& 4.75 & 7.92 & 25.31 & 3.33 & 3.16 & 23.86 & 3.04 & 2.89 & 47.35 & 19.94 & 17.30 & 21.64 & 15.41 \\

& 5 & 6.25 & 25.31 & 5.00 & 3.96 & 22.85 & 2.56 & 2.26 & 47.35 & 19.21 & 17.26 & 21.37 & 15.20 \\

\midrule

\multirow[c]{20}{*}{\textbf{Low}}
& 0.25 & 2.08 & 15.31 & \firstckpt{2.50} & 1.27 & \firstckpt{16.41} & 1.01 & 0.57 & 37.52 & 16.45 & 12.15 & 16.06 & 10.53 \\

& 0.5 & 1.67 & \firstckpt{18.44} & 0.83 & \firstckpt{3.16} & 15.28 & \firstckpt{1.85} & \firstckpt{1.51} & 42.80 & 16.68 & 12.70 & 17.59 & 11.49 \\

& 0.75 & \firstckpt{3.33} & 15.31 & 0.83 & 1.74 & 15.34 & 0.30 & 0.75 & 40.20 & 16.77 & 11.54 & 16.38 & 10.61 \\

& 1 & 1.25 & \firstckpt{18.44} & 1.25 & 2.53 & 15.59 & 1.13 & 1.13 & \firstckpt{44.47} & \firstckpt{17.37} & \firstckpt{13.22} & \firstckpt{18.11} & \firstckpt{11.64} \\

\cdashline{2-14}

& 1.25 & 2.92 & 18.12 & 2.08 & 2.53 & 18.56 & 1.19 & \best{1.57} & 45.02 & 18.57 & 14.30 & 19.04 & 12.49 \\

& 1.5 & 4.17 & 22.50 & 2.08 & 3.48 & 17.87 & 0.95 & 1.13 & 47.10 & 20.27 & 15.02 & 19.94 & 13.46 \\

& 1.75 & 4.58 & 17.19 & 2.08 & 1.74 & 16.73 & 0.54 & 1.01 & 44.40 & 17.83 & 13.46 & 18.27 & 11.96 \\

& 2 & 4.17 & 17.19 & 2.50 & 3.32 & 16.73 & 0.89 & 0.57 & 47.35 & 19.21 & 15.57 & 19.78 & 12.75 \\

& 2.25 & 5.00 & 21.25 & 2.08 & 4.75 & 19.70 & 0.95 & 1.07 & 47.85 & 19.99 & 16.39 & 20.68 & 13.90 \\

& 2.5 & 4.17 & 20.00 & 2.08 & 2.22 & 17.99 & 0.95 & 1.01 & 46.08 & 19.44 & 15.72 & 19.73 & 12.97 \\

& 2.75 & 3.75 & 19.38 & 2.08 & 3.01 & 17.61 & 1.73 & 1.32 & 47.17 & \best{20.82} & 15.91 & 20.28 & 13.28 \\

& 3 & 5.42 & \best{24.38} & 0.00 & 4.59 & 19.26 & 1.55 & \best{1.57} & 46.88 & 20.68 & 16.11 & 20.55 & 14.04 \\

& 3.25 & 6.25 & 18.75 & 0.83 & 4.27 & 17.11 & 1.55 & 1.01 & 46.90 & 20.13 & 16.04 & 20.14 & 13.28 \\

& 3.5 & 5.83 & 20.00 & 2.08 & \best{7.12} & 19.95 & 1.55 & 0.63 & 46.88 & 19.03 & 16.11 & 20.37 & 13.92 \\

& 3.75 & 6.25 & 22.81 & 0.83 & 4.11 & 19.63 & \best{1.90} & 0.82 & 47.40 & 20.45 & 16.04 & 20.59 & 14.03 \\

& 4 & \best{7.50} & 22.81 & 2.92 & 4.59 & 17.93 & 1.79 & 1.07 & 46.88 & 19.44 & 15.98 & 20.26 & 14.09 \\

& 4.25 & 5.42 & 20.00 & \best{3.75} & 4.27 & 19.89 & 1.43 & 1.01 & 48.02 & 18.52 & 15.96 & 20.45 & 13.83 \\

& 4.5 & 7.08 & 19.38 & 1.67 & 4.59 & \best{20.52} & 1.13 & 0.50 & \best{48.12} & 19.07 & 16.19 & 20.59 & 13.82 \\

& 4.75 & 6.25 & 20.94 & \best{3.75} & 5.54 & 19.38 & 1.25 & 1.38 & 48.08 & 19.53 & \best{16.63} & \best{20.84} & \best{14.27} \\

& 5 & 6.25 & 20.62 & 1.67 & 4.91 & 17.93 & 1.19 & 1.19 & 47.33 & 19.58 & 15.52 & 20.13 & 13.62 \\

\bottomrule

\end{tabular}%

}

\end{table*}

\begin{table*}[t]

\centering
\small
\setlength{\tabcolsep}{4pt}
\renewcommand{\arraystretch}{1.10}

\caption{
\textbf{Llama-3.1-8B results across training epochs under high- and low-likelihood data selection.}
\ \legendbox{firstcyan}\ \textbf{Cyan} marks the best value within the early stage (epochs 0.25--1),
\ \legendbox{bestyellow}\ \textbf{yellow} marks the best value across epochs 0.25--5 within each setting, and
\ \legendbox{bothgreen}\ \textbf{green} marks cells that satisfy both.
Dashed rules separate the earliest stage from later training.
}

\label{tab:llama31_8b_high_low}

\resizebox{\textwidth}{!}{%

\begin{tabular}{llrrrrrrrrrrrr}

\toprule

\textbf{Setting} &
\textbf{Epoch} &
\textbf{AIME24} &
\textbf{AMC} &
\textbf{CNM} &
\textbf{GK} &
\textbf{GPQA} &
\textbf{GSM} &
\textbf{KY} &
\textbf{MATH} &
\textbf{MNV} &
\textbf{OB} &
\textbf{W.\ Avg} &
\textbf{Avg} \\

\midrule

\multirow[c]{20}{*}{\textbf{High}}
& 0.25 & \firstckpt{4.58} & 19.38 & 2.92 & 3.48 & 28.85 & 6.61 & 1.38 & 43.80 & \firstckpt{23.81} & \firstckpt{14.74} & 21.04 & 14.95 \\

& 0.5 & 2.92 & 17.19 & 1.25 & 5.38 & 29.10 & \firstckpt{10.54} & \firstckpt{4.65} & \firstckpt{43.85} & 21.97 & 13.74 & 21.19 & 15.06 \\

& 0.75 & 2.92 & 17.19 & \firstckpt{3.75} & 3.32 & 27.65 & 5.60 & 3.08 & 43.08 & 21.28 & 13.46 & 20.07 & 14.13 \\

& 1 & 2.92 & \firstckpt{20.00} & 1.67 & \firstckpt{5.54} & \firstckpt{33.40} & 7.38 & 3.64 & 43.15 & 22.15 & 14.65 & \firstckpt{21.38} & \firstckpt{15.45} \\

\cdashline{2-14}

& 1.25 & 5.00 & 21.56 & 4.17 & 9.18 & 36.24 & 12.20 & 5.09 & 46.50 & 23.58 & 17.04 & 24.08 & 18.05 \\

& 1.5 & 4.58 & 18.12 & 2.92 & 6.33 & 36.81 & 11.19 & 5.28 & 48.65 & 25.00 & 15.63 & 24.10 & 17.45 \\

& 1.75 & 2.50 & 19.69 & 1.67 & 7.12 & 35.73 & 10.18 & 5.90 & 47.90 & 22.61 & 16.35 & 23.73 & 16.97 \\

& 2 & 5.42 & 22.81 & 3.33 & 7.59 & 35.48 & 10.60 & 5.15 & 49.73 & 23.16 & 17.93 & 24.77 & 18.12 \\

& 2.25 & 5.00 & 19.06 & 2.92 & 12.34 & 35.04 & 13.39 & \best{7.04} & 48.20 & 22.98 & 16.76 & 24.54 & 18.27 \\

& 2.5 & 5.83 & 22.19 & 2.50 & 13.13 & 31.82 & 15.48 & 6.28 & 49.40 & 23.12 & 18.35 & 25.24 & 18.81 \\

& 2.75 & 5.42 & \best{25.31} & 1.25 & 11.71 & 35.16 & 14.94 & 6.72 & 48.30 & 23.71 & 17.57 & 25.10 & 19.01 \\

& 3 & 5.83 & 23.75 & 2.92 & 12.66 & 35.86 & 13.39 & 5.97 & 51.25 & 22.52 & \best{19.24} & 26.00 & 19.34 \\

& 3.25 & 3.33 & 18.44 & 3.75 & \best{15.03} & 34.72 & 15.42 & 6.91 & 50.10 & \best{25.09} & 18.20 & 25.88 & 19.10 \\

& 3.5 & \best{8.33} & 23.12 & 2.92 & 14.24 & 35.48 & \best{16.19} & 5.72 & 50.48 & 23.39 & 18.57 & 26.01 & 19.84 \\

& 3.75 & 5.42 & 22.19 & 5.00 & \best{15.03} & \best{38.57} & 14.70 & 6.09 & 50.65 & 24.08 & \best{19.24} & \best{26.51} & \best{20.10} \\

& 4 & 4.58 & 23.44 & 6.25 & 14.40 & 35.80 & 15.30 & 6.97 & 50.40 & 23.81 & 18.65 & 26.13 & 19.96 \\

& 4.25 & 5.42 & 23.12 & \best{6.67} & 14.87 & 34.91 & 14.40 & 6.34 & 50.50 & 23.16 & 18.72 & 25.91 & 19.81 \\

& 4.5 & 5.83 & 24.06 & 5.83 & 13.61 & 36.81 & 13.04 & 5.46 & \best{51.28} & 24.63 & 18.44 & 26.10 & 19.90 \\

& 4.75 & 4.58 & 22.19 & 5.42 & 13.92 & 36.81 & 14.76 & 6.28 & 51.25 & 23.30 & 18.67 & 26.19 & 19.72 \\

& 5 & 5.00 & 24.06 & 3.75 & 14.72 & 35.98 & 15.36 & 5.84 & 50.98 & 24.72 & 18.24 & 26.16 & 19.87 \\

\midrule

\multirow[c]{20}{*}{\textbf{Low}}
& 0.25 & \firstckpt{2.92} & 14.37 & 0.83 & 6.33 & 26.64 & 8.81 & 3.39 & 42.62 & \firstckpt{23.71} & 14.50 & 20.84 & 14.41 \\

& 0.5 & 1.25 & 11.88 & 1.67 & 6.96 & 26.33 & \firstckpt{9.58} & \firstckpt{4.65} & 41.70 & 22.06 & 13.54 & 20.26 & 13.96 \\

& 0.75 & 2.08 & 14.06 & \firstckpt{2.50} & \firstckpt{7.59} & 30.05 & 9.29 & 4.02 & 43.38 & 22.33 & 14.46 & 21.28 & 14.98 \\

& 1 & 1.67 & \firstckpt{17.19} & 1.67 & 6.49 & \firstckpt{30.43} & 8.63 & 2.51 & \firstckpt{46.80} & 22.89 & \firstckpt{15.33} & \firstckpt{22.22} & \firstckpt{15.36} \\

\cdashline{2-14}

& 1.25 & 3.75 & 19.38 & 3.75 & 13.92 & 33.96 & 13.04 & 6.22 & 48.62 & 25.92 & 16.33 & 24.71 & 18.49 \\

& 1.5 & 4.17 & 20.94 & 4.17 & 11.71 & 33.08 & 13.45 & 4.96 & 48.98 & 25.41 & 17.81 & 24.99 & 18.47 \\

& 1.75 & 3.33 & 21.88 & 3.33 & 9.34 & 32.32 & 10.24 & 4.27 & 48.08 & 24.40 & 16.41 & 23.72 & 17.36 \\

& 2 & 3.33 & 21.88 & 1.25 & 6.96 & 30.81 & 10.42 & 3.64 & 50.95 & 25.41 & 17.11 & 24.41 & 17.18 \\

& 2.25 & 2.08 & 21.25 & 4.17 & 11.23 & 35.29 & 13.21 & 3.89 & 50.12 & 25.23 & 18.56 & 25.49 & 18.50 \\

& 2.5 & 5.42 & 21.25 & 1.67 & 13.61 & 36.30 & 15.48 & \best{6.47} & 51.25 & 25.23 & 19.26 & 26.58 & 19.59 \\

& 2.75 & 3.75 & 23.75 & 2.50 & 13.13 & 36.17 & 14.40 & 4.90 & 51.80 & 24.54 & 17.83 & 25.95 & 19.28 \\

& 3 & 5.83 & 27.50 & 2.50 & 12.50 & 35.98 & 16.19 & 6.16 & 51.42 & 25.51 & 18.54 & 26.53 & 20.21 \\

& 3.25 & 4.58 & 25.00 & 4.17 & 13.61 & 35.23 & 16.85 & 5.53 & 51.32 & 24.63 & 19.04 & 26.49 & 20.00 \\

& 3.5 & \best{7.50} & 21.88 & 2.92 & 14.56 & \best{37.18} & 17.02 & 4.77 & 52.52 & 26.15 & 18.59 & 26.94 & 20.31 \\

& 3.75 & 5.42 & 23.75 & 5.00 & 17.41 & 35.23 & 17.50 & 5.21 & \best{52.95} & 26.19 & 19.52 & \best{27.36} & 20.82 \\

& 4 & 5.00 & 23.12 & 4.58 & 14.56 & 34.72 & 17.14 & 5.53 & 52.45 & 26.33 & 18.69 & 26.84 & 20.21 \\

& 4.25 & 2.50 & 22.50 & 2.50 & \best{19.62} & 35.23 & 17.80 & 4.71 & 51.78 & \best{26.84} & 19.70 & 27.20 & 20.32 \\

& 4.5 & 4.17 & \best{27.81} & \best{6.25} & 13.77 & \best{37.18} & 17.14 & 5.09 & 52.80 & 25.46 & 19.63 & 27.35 & \best{20.93} \\

& 4.75 & 4.17 & 25.94 & 3.75 & 17.56 & 33.52 & 17.02 & 5.21 & 52.92 & 25.23 & 19.56 & 27.07 & 20.49 \\

& 5 & 4.17 & 22.19 & 4.58 & 15.19 & 34.28 & \best{18.15} & 5.34 & 51.30 & 25.97 & \best{20.15} & 27.02 & 20.13 \\

\bottomrule

\end{tabular}%

}

\end{table*}